\documentclass{article}

\PassOptionsToPackage{numbers, compress}{natbib}
\usepackage[preprint]{neurips_2023}
\usepackage[utf8]{inputenc} % allow utf-8 input
\usepackage[T1]{fontenc}    % use 8-bit T1 fonts
\usepackage{hyperref}       % hyperlinks
\usepackage{url}            % simple URL typesetting
\usepackage{booktabs}       % professional-quality tables
\usepackage{amsfonts}       % blackboard math symbols
\usepackage{nicefrac}       % compact symbols for 1/2, etc.
\usepackage{microtype}      % microtypography
\usepackage{xcolor}         % colors
\usepackage{graphicx}

\usepackage{algorithm}  
\usepackage{algorithmic}  
\usepackage{wrapfig}
\usepackage{multirow}
\usepackage{amsmath}
\usepackage{amssymb}
\usepackage{mathtools}
\usepackage{amsthm}
\usepackage{dsfont}
\usepackage{verbatim}
\usepackage{threeparttable}

\newtheorem{theorem}{Theorem}

\newtheorem{corollary}{Corollary}
\newtheorem{definition}{Definition}

\newcommand{\indep}{\perp \!\!\! \perp}
\newcommand{\notindep}{\not \! \perp \!\!\! \perp}

\title{Hierarchical Topological Ordering with Conditional Independence Test for Limited Time Series}

\author{ 
Anpeng Wu\\
Zhejiang University\\
\texttt{anpwu@zju.edu.cn}\\
\And
Haoxuan Li\\
Peking University\\
\texttt{hxli@stu.pku.edu.cn}\\
\And
Kun Kuang\\
Zhejiang University\\
\texttt{kunkuang@zju.edu.cn}\\
\And
Keli Zhang\\
Huawei Noah's Ark Lab\\
\texttt{zhangkeli1@huawei.com}\\
\And
Fei Wu\\
Zhejiang University\\
\texttt{wufei@zju.edu.cn}\\
}

\begin{document}

\maketitle

\begin{abstract}
    Learning directed acyclic graphs (DAGs) to identify causal relations underlying observational data is crucial but also poses significant challenges. Recently, topology-based methods have emerged as a two-step approach to discovering DAGs by first learning the topological ordering of variables and then eliminating redundant edges, while ensuring that the graph remains acyclic. However, one limitation is that these methods would generate numerous spurious edges that require subsequent pruning. To overcome this limitation, in this paper, we propose an improvement to topology-based methods by introducing limited time series data, consisting of only two cross-sectional records that need not be adjacent in time and are subject to flexible timing. By incorporating conditional instrumental variables as exogenous interventions, we aim to identify descendant nodes for each variable. Following this line, we propose a hierarchical topological ordering algorithm with conditional independence test (HT-CIT), which enables the efficient learning of sparse DAGs with a smaller search space compared to other popular approaches. The HT-CIT algorithm greatly reduces the number of edges that need to be pruned. Empirical results from synthetic and real-world datasets demonstrate the superiority of the proposed HT-CIT algorithm.
\end{abstract}

\section{Introduction}
\label{sec:intro}

Learning causal relations from observational data is crucial across various scientific disciplines, such as epidemiology~\cite{vandenbroucke2016causality}, economics~\cite{pearl2009causality}, biology~\cite{triantafillou2017predicting}, and social science~\cite{malinsky2018causal}. 
This enables researchers to make informed decisions and deepen their understanding of the underlying causal structure of the data~\cite{peters2017elements,pearl2009causality}. 
Traditionally, constraint-based methods \cite{spirtes2000causation,zhang2008completeness,ramsey2012adjacency} use conditional independence tests (CIT) to identify causal relations, score-based methods \cite{tsamardinos2006max,ke2019learning,zhu2020causal} search through the space of all possible causal structures with the aim of optimizing a specified metric, and continuous-optimization methods~\cite{zheng2018NOTEARS, Lachapelle2020Gradient-Based} view the search as a constrained optimization problem and apply first-order optimization methods to solve it. 
However, discovering the underlying DAG with a greedy combinatorial optimisation method can be expensive and challenging due to the super-exponential growth of the set of DAGs with the number of nodes~\cite{teyssier2005ordering, rolland2022score}. 

Recently, topology-based methods~\cite{teyssier2005ordering,peters2014causal, loh2014high, park2017bayesian, ghoshal2018learning} develop a two-stage method to speed up the combinatorial search problem over the space of DAGs, under Gaussian additive models.
Firstly, they learn a topological ordering of the nodes, in which a node in the ordering can only be a parent to nodes that appear after it in the same ordering; then, the target DAG is constructed by adhering to the topological ordering and pruning any unnecessary edges. Once a topological order is established, the acyclicity constraint is automatically upheld without further optimisation~\cite{buhlmann2014cam, rolland2022score}. 
Nevertheless, single cross-sectional data alone is generally not efficient for identifying the DAG \cite{yang2018characterizing}.  
In line with these work, SCORE and DiffAN \cite{rolland2022score, sanchez2022diffusion} uses the Hessian of the data log-likelihood to iteratively identify and remove leaf nodes to find a complete topological ordering (Fig.~\ref{fig:figure1}(b)). They restrict the number and direction of possible edges in the learned DAG, but typically creates a topological ordering with many spurious edges that must be pruned. 
Additionally, it should be noted that they may not necessarily yield a unique solution.
This pose potential difficulties for pruning spurious edges and result in errors for learning DAGs, reflected in the gradually deteriorating performance of the SCORE method as the number of nodes increases in the experiment section.

\vspace{-2pt}
Fortunately, it is feasible to acquire two temporal data slices within a brief duration, which help topological ordering for learning Directed Acyclic Graphs. In such cases, as shown in Fig.~\ref{fig:figure1}(a), we assume the 
previous state of a node only affects its own state and its descendants in current time, and the causality on nodes remain invariant. It is common in real-world application,  such as in power systems, interrelationships between malfunctions and their associated operations remain unchanged and only be subject to a time-lagged effect from their own previous state, over a brief period.
Then, any perturbation of the previous state on a node is transmitted to itself and its descendant nodes at the current moment, which can be regarded as a conditional instrumental variable (CIV). Based on this, this paper demonstrates a single conditional independence test per variable is sufficient to build a more efficient unique hierarchical topological ordering with merely a few spurious edges (Fig.~\ref{fig:figure1}(c)). The search space over the learned hierarchical topological ordering is much smaller than that of SCORE \cite{rolland2022score}. Then the spurious edges to descendant nodes can be pruned by feature selection algorithms, such as CAM \cite{buhlmann2014cam}, to directly yield an asymptotic directed acyclic graph.

%%%%%%%%%%%%%%%%%%%%%%%%%%%%%%%%%%%%%%%%
%%%%%%%%%%%%%%%%%%%%%%%%%%%%%%%%%%%%%%%%

\begin{figure}[t]
\centering
\includegraphics[width=0.92\columnwidth]{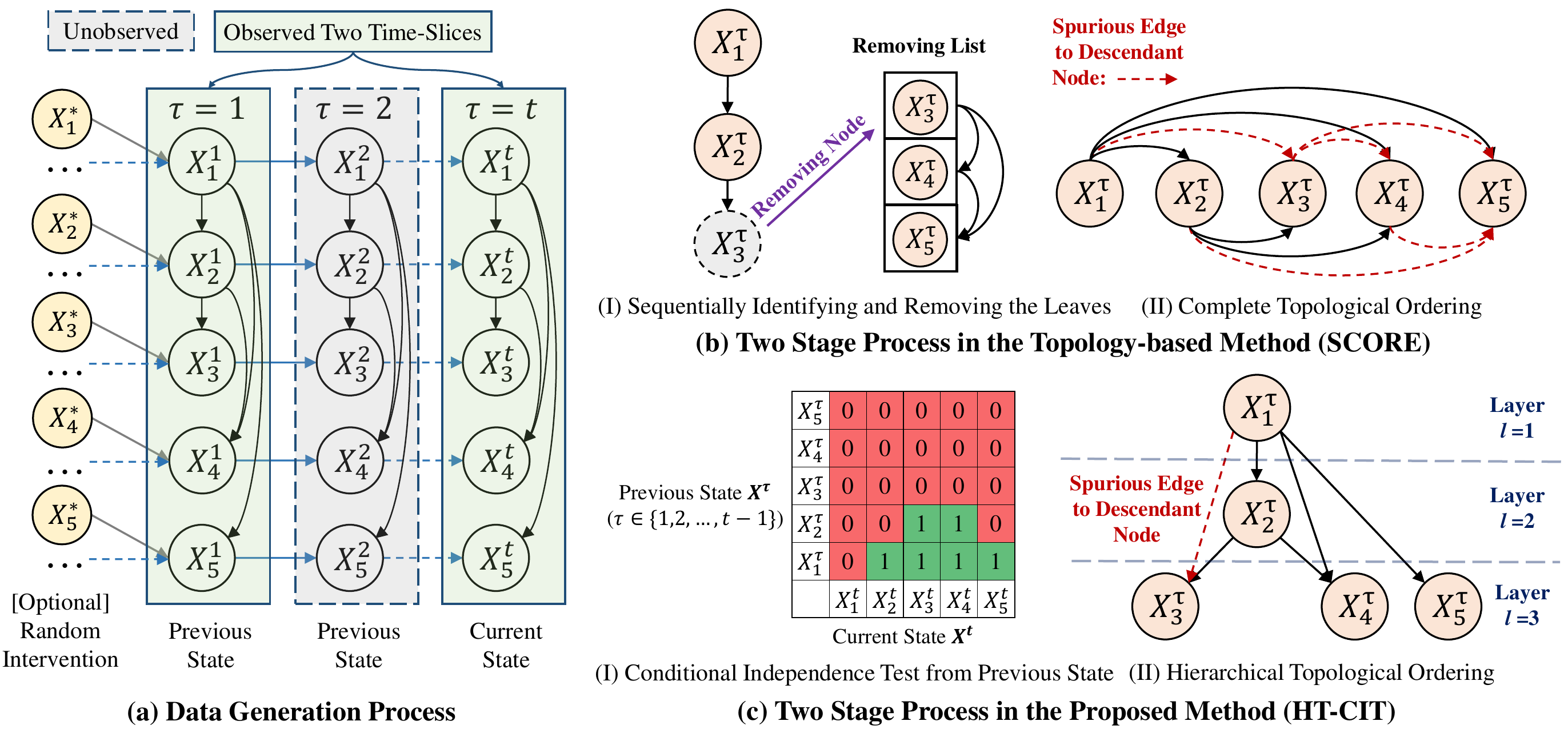}
\vspace{-4pt}
\caption{Two time-slices environment and comparison of SCORE and HT-CIT architecture. }
\vspace{-13pt}
\label{fig:figure1}
\end{figure}

%%%%%%%%%%%%%%%%%%%%%%%%%%%%%%%%%%%%%%%%
%%%%%%%%%%%%%%%%%%%%%%%%%%%%%%%%%%%%%%%%

\vspace{-2pt}
In general, the identification result of hierarchical topological ordering in this paper arises not from a multivariate continuous Additive Noise Model (ANM) but rather from the independence properties of auxiliary variables in previous state. In such cases, in which a previous state is observed or a random intervention is implemented for each variable (Fig.~\ref{fig:figure1}(a)), we propose a hierarchical topological ordering algorithm with conditional independence test to accurately identify the underlying DAG by adhering to the learned unique topological ordering and pruning unnecessary edges. This improve upon existing methods by significantly simplifying the process of finding a topological ordering and providing a more efficient topological ordering with high quality. 

\vspace{-2pt}
The main contributions in this paper are as follows:
\vspace{-4pt}
\begin{itemize} 
  \item Theoretically, we prove that if we observe a previous state or implement a random intervention for each variable, a single conditional independence test per variable is sufficient to distinguish between its descendants and non-descendant nodes in topological ordering. 
  \item Based on this, we propose HT-CIT, a novel identifiable topological sorting algorithm for a unique hierarchical topological ordering. The search space for underlying DAGs over the learned hierarchical topological ordering is much smaller than that of SCORE.  
  \item The empirical experiments demonstrate the superiority of our algorithm in synthetic data. The detected causality in real CMR application is in accordance with the consensus and can provide new insights for personalized policy-decision. 
\end{itemize}

\section{Related work}

Conventional methods rely on certain assumptions to uncover the true underlying DAG, including causal Markov condition, faithfulness, causal sufficiency, and additive noise models. Constraint-based methods \cite{spirtes2000causation,zhang2008completeness,ramsey2012adjacency} typically rely on conditional independence tests to identify causal relationships by testing the independence between variables given a set of conditions~\cite{spirtes2000causation, sun2007kernel, hyttinen2014constraint}. By testing different sets of conditions, these methods can identify the causal relationships between variables by determining which variables are dependent or independent. Examples of constraint-based methods include PC, FCI, SGS, and ICP~\cite{spirtes2000causation, peters2016causal}. Besides, score-based methods \cite{tsamardinos2006max,ke2019learning,zhu2020causal} search through the space of all possible causal structures with the aim of optimizing a specified metric, and rely on local heuristics to enforce the acyclicity, such as GES, and GIES~\cite{chickering2002optimal,hauser2012characterization}. Continuous-optimization methods ~\cite{zheng2018NOTEARS, Lachapelle2020Gradient-Based} view the search as a constrained optimization problem and apply first-order optimization methods to solve it, such as GraNDAG, GOLEM, NOTEARS, ReScore~\cite{Lachapelle2020Gradient-Based, ng2020golem, zheng2018dags, zheng2020learning, zhang2023boosting}. Despite their strengths, existing approaches have limitations. Specifically, these methods may only find causal structures within an equivalence class, resulting in a limited understanding of the underlying causal relationships. Besides, they also rely on local heuristics for enforcing acyclicity constraints, which can be insufficient for handling large datasets effectively.

There are some hybrid methods that combine the advantages of both types of methods \cite{tsamardinos2006max, chen2021fritl, li2022hybrid, hasan2023survey}. For example, GSP and IGSP algorithms \cite{solus2021consistency,wang2017permutation} evaluate the score of each DAG structure using some information criterion (such as Bayesian information criterion) and searches for the optimal solution by iteratively changing permutations. Additionally, topology-based methods tackle the causal discovery problem by finding a certain topological ordering of the nodes and then pruning the spurious edges in topological ordering~\cite{teyssier2005ordering,peters2014causal, loh2014high, park2017bayesian, ghoshal2018learning, ahammad2021new, sanchez2022diffusion, reisach2023simple}. Examples of topology-based methods include CAM, SCORE and NoGAM~\cite{buhlmann2014cam,rolland2022score,montagna2023causal}. These methods have a less combinatorial problem than other methods, as the set of permutations is much smaller than the set of DAGs. Once a topological order is fixed, the acyclicity constraint is naturally enforced, making the pruning step easier to solve. These methods restricts the number and direction of possible edges in the learned DAG, but typically creates many spurious edges that must be pruned.

To the best of our knowledge, most of traditional time series methods \cite{granger1969investigating, nauta2019causal, bussmann2021neural, lowe2022amortized, assaad2022survey} in causal discovery typically require data observed at a series of time points, covering more than two cross-sectional, and unequal time intervals may lead to causal misidentification. Different with traditional time-series task, we focus on two time-slices for learning DAGs and concatenate two time-slices data with the temporal edge to a single cross-sectional dataset with only one causal graph. Our proposed HT-CIT algorithm is a joint constraint-based and topology-based method that utilizes conditional independence tests with instrumental variables to distinguish between descendant and non-descendant nodes and build a topological graph. By reducing the search space of the underlying DAG, HT-CIT provides a more accurate and efficient solution for causal discovery compared to traditional methods. With knowledge of a cause for each variable, HT-CIT is a promising approach for causal discovery.

% \section{Preliminaries}
% \label{sec:prel}

% \section{Background}

% \subsection{Problem Setup}

\section{Problem setup}

We consider the problem of discovering the causal structure between $d$ variables from two time-slices, when observations are collected over a short period of time. Let $\boldsymbol{X} = \{{X}_{i}^{\tau}\}_{d \times t}$ be a multivariate time series with $d$ variables and $t$ time steps, and the observation of a time series variable ${X}_{i}$ at time $\tau$ is denoted by ${X}_{i}^{\tau}$, where $i=1,2,\cdots,d$ and $\tau=1,2,\cdots,t$. We assume that the true causal structure is represented by a DAG $\mathcal{G}$. 
% In the underlying DAG $\mathcal{G}=(V, E)$, each node $i \in V$ corresponds to the random variable $\boldsymbol{X}_i$ and an edge $(i, j) \in E$ represents a direct causal relation from variable $\boldsymbol{X}_i$ to $\boldsymbol{X}_j: \boldsymbol{X}_i \rightarrow \boldsymbol{X}_j$.  
For each ${X}_{i}^{\tau}$, we use the notation $\mathbf{pa}_{i}^{\tau}$ to refer to the set of parents of ${X}_{i}$ at time-slice $\tau$ in $\mathcal{G}$. Similarity, we define $\mathbf{ch}_{i}^{\tau}$ for the set of child nodes, $\mathbf{an}_{i}^{\tau}$ for the set of ancestors, $\mathbf{sib}_{i}^{\tau}$ for the set of siblings, and $\mathbf{de}_{i}^{\tau}$ for the set of descendants.
As shown in Fig.~\ref{fig:figure1}(a), under \emph{causal Markov condition}, we assume that the structure of the graph can be expressed in the functional relationship, for any node $i=1, 2, \cdots, d$ at time $\tau = 1,2,\cdots, t$: 
\begin{eqnarray}
\label{eq:data}
X_i^{\tau}=f_i\left(\mathbf{pa}_{i}^{\tau}\right)+g_i\left(X_i^{\tau-1}\right)+\epsilon_i^{\tau},
\end{eqnarray} 
where, $f_i\left(\mathbf{pa}_{i}^{\tau}\right)$ is a twice continuously differentiable function in each component, which encode the \emph{instantaneous effect} from its parents $\mathbf{pa}_{i}^{\tau}$; $g_i$ encode the \emph{time-lagged effect} from its previous state $X_i^{\tau-1}$, which must contain a non-zero \emph{1st Order Autoregressive Lagged Effect}; and $\epsilon_i^{\tau}$ is additive noise variables independently drawn from an identical distribution, i.e., \emph{Additive Noise Models}.

% Then, following Eq.~\ref{eq:data}, the joint distribution $Pr(\boldsymbol{X}^{\tau})$ can be factorized as:
% \begin{eqnarray*}
% Pr(\boldsymbol{X}^{\tau})=Pr(\boldsymbol{X}^{\tau-1})\prod_{i=1}^d Pr\left(X_i^{\tau} \mid \mathbf{pa}_{i}^{\tau}, X_i^{\tau-1}\right) \nonumber = Pr(\boldsymbol{X}^{0})\prod_{j=1}^{\tau}\prod_{i=1}^d Pr\left(X_i^{j} \mid \mathbf{pa}_{j}^{\tau}, X_i^{j-1}\right),
% \end{eqnarray*} 
% where $Pr\left(X_i^{\tau} \mid \mathbf{pa}_{i}^{\tau}, X_i^{\tau-1}\right)$ denotes the variable $X_i^{\tau}$ is generated by its parent nodes $\mathbf{pa}_{i}^{\tau}$ and its previous stage $X_i^{\tau-1}$. 

\begin{definition}[1st Order Autoregressive Lagged Effect, Definition 9 in \cite{hasan2023survey}]
\label{def:ALE} 
    When the delay a variable's previous lagged value on its current value is a unit of time, then it is known as 1st Order Autoregressive Lagged Effect.
\end{definition}
\vspace{-4pt}

\textbf{Two time-slices data (with optional random intervention)}. To the best of our knowledge, as stated in Definition 7 in \cite{hasan2023survey}, traditional time series methods require a series of time-slices and the observations in multiple time series should be collected over consistent intervals of time. However, as illustrated in Fig.~\ref{fig:figure1}(a), a more practical scenario is that we are only given two (non-)consecutive time-slices, which arise from a cross-section of an arbitrary previous moment ($\boldsymbol{X}^{\tau}$, $\tau \in \{1, 2, ..., t-1\}$) as well as a time slice of the current moment ($\boldsymbol{X}^{t}$). Given these observations $\mathcal{D} = \{ \boldsymbol{X}^{\tau}, \boldsymbol{X}^{t} \}_{\tau < t}$, the task is to identify the underlying DAG (rather than Markov equivalence class) of the causal structure. In addition, it is possible to implement optional random interventions for each variable in the previous state, if any, which will further help us to improve the efficiency of finding topological ordering. 

In general, the identification result of hierarchical topological ordering arises from the conditional independence properties of auxiliary variables in previous state, which can be regarded as \emph{conditional instrumental variable}\footnote{Given $\mathbf{pa}_{i}^{\tau}$, the auxiliary variable $X_i^\tau$ (i.e., conditional instruments) is conditional independent with its non-descendants $\{\mathbf{an}_{i}^{\tau}, \mathbf{sib}_{i}^{\tau}\}$ and only indirectly affects the its descendant nodes $\mathbf{de}_i^t$ (i.e., outcomes) at the current state through its association with $X_i^t$ (i.e., treatments).} (CIV). To discriminate between the descendant and non-descendant nodes of each variable using the conditional independence property, except for standard \emph{causal Markov condition}, \emph{faithfulness} and \emph{causal sufficient} assumptions \cite{peters2014causal}, it is necessary to ensure that the time series remains consistent over time and that the summary graph is acyclic. This is a common practice when studying stationary causal relationships in short-term.

\begin{definition}[Consistency Throughout Time, Definition 7 in \cite{assaad2022survey})]
\label{def:CTT}
    A causal graph $\mathcal{G}$ for a multivariate time series $\boldsymbol{X}$ is said to be consistent throughout time if all the causal relationships remain constant in direction throughout time, while allowing for variability in causal effects.
\end{definition}
\vspace{-4pt}

\begin{definition}[Acyclic Summary Causal Graph, Definition 11 in \cite{assaad2022survey})]
\label{def:ASCG}
    The summary causa graph of a multivariate time series is considered acyclic if the lagged effect of each variable solely affects its own value and its descendants, without any influence on its non-descendants at the current time.
\end{definition}
\vspace{-4pt}

Following these assumption, this model is known to be identifiable from observational data \cite{peters2014causal, buhlmann2014cam}, meaning that it is possible to recover the instantaneous DAG underlying the generative model (Eq.~\eqref{eq:data}). 
In the present work, we will utilize two time-slice data to aid in learning a unique hierarchical topological ordering with a smaller search space compared to other advanced approaches.

\section{Algorithm}
\label{sec:theorem}

In this section, we will introduce the complete topological ordering from classical topology-based approaches \cite{rolland2022score, sanchez2022diffusion} and show how two time-slice data help identify a unique hierarchical topological ordering. 
We first propose HT-CIT, a novel identifiable topological sorting algorithm for hierarchical topological ordering, which applicable to any types of noise. The search space over the learned hierarchical topological ordering is much smaller than that of SCORE. 
Then, the underlying DAGs can be found by pruning the unnecessary edges with a well-defined pruning method \cite{buhlmann2014cam, Lachapelle2020Gradient-Based}.

% SCORE and DiffAN \cite{rolland2022score, sanchez2022diffusion}

\subsection{From complete to hierarchical topological ordering}

As illustrated in Fig.~\ref{fig:figure1}(b), the conventional typology-based approach SCORE \cite{rolland2022score,sanchez2022diffusion} perform sequential identification and removal of leaf nodes to generate a complete topological ordering based on the Hessian's diagonal of the data log-likelihood, which often contains many spurious edges.

\begin{definition}[Complete Topological Ordering]
\label{def:CTO}
    The complete topological ordering ($\pi(\boldsymbol{X}) = ( X_{\pi_1}, X_{\pi_2}, \cdots, X_{\pi_d} )$, $\pi_i$ is the reordered index of node) refers to a sorting of all nodes in a DAG such that for any pair of nodes $X_{\pi_i}$ and $X_{\pi_j}$, if $i < j$, then there is a directed edge from $X_{\pi_i}$ to $X_{\pi_j}$. 
\end{definition}
\vspace{-4pt}

% Once a topological order is established, the acyclicity constraint is automatically upheld without further optimisation. 
However, a complete topological ordering with $d(d-1)/2$ edges is a dense graph that contains numerous spurious edges, many of which point to non-descendants unnecessarily. Moreover, these methods \cite{rolland2022score,sanchez2022diffusion} may not always produce a unique solution, making it challenging to eliminate false edges and resulting in errors when learning DAGs. Fortunately, obtaining two time-slice data within a brief duration cab help identify a unique hierarchical topological ordering (Fig.~\ref{fig:figure1}(c)), in which each edge only points from an ancestor node to its descendant nodes and not to any non-descendant nodes.

\begin{definition}[Hierarchical Topological Ordering]
\label{def:HTO}
In the hierarchical topological ordering e.g., $\Pi(\boldsymbol{X}) = (\{X_{\pi_1}\}_{\boldsymbol{L}_1}, \{X_{\pi_2}, X_{\pi_3}\}_{\boldsymbol{L}_2}, \{X_{\pi_4}, X_{\pi_5}\}_{\boldsymbol{L}_3}, \cdots) $, variables at the same layer are grouped together. Each layer is denoted by $\boldsymbol{L}_i$ and we represent $l_j$ as the located layer of $X_j$. If there is a directed edge from $X_{\pi_i}$ to $X_{\pi_j}$, then $X_{\pi_i}$ is located in a higher layer than $X_{\pi_j}$, i.e., $l_{\pi_i} > l_{\pi_j}$.
\end{definition}
\vspace{-4pt}

Notably, there are multiple different hierarchical topological orderings corresponding to a same DAG. The complete topological ordering also is a special case of the hierarchical topological ordering with each layer consisting of only one node. To obtain a unique hierarchical topological ordering and efficiently implement causal discovery,  we identify the descendant nodes of each variable, and link direct edges to its descendant nodes to build the unique topological ordering on $\boldsymbol{X}^{t}$ (Fig.~\ref{fig:figure1}(c)):
\begin{eqnarray}
\label{eq:desgraph}
X_i^t \rightarrow \mathbf{de}_i^t, \text{ for } i=1,2,\cdots,d. 
\end{eqnarray} 

In the unique hierarchical topological ordering, each node only has a directed edge to each of its descendants and does not point to non-descendants in the ordering. Nest, we will introduce how two time-slice data help identify the topological ordering. 

\subsection{Two time-slices help identify hierarchical topological ordering}

In a consistent time series data throughout time with acyclic summary causal graph, suppose that we are given an observational data with two (non-)consecutive time-slices $\mathcal{D} = \{ \boldsymbol{X}^{\tau}, \boldsymbol{X}^{t} \}_{\tau < t}$, a single conditional independence test per variable is sufficient to distinguish between its descendants and non-descendant nodes in topological ordering. 

\begin{theorem}[Descendant-oriented Conditional Independence Criteria]
\label{th:our}
    Given a two time-slice observations $\mathcal{D} = \{ \boldsymbol{X}^{\tau}, \boldsymbol{X}^{t} \}_{\tau < t}$ from time series satisfying Defs. \ref{def:ALE}, \ref{def:CTT} and \ref{def:ASCG}, for the variables $X_i^{\tau}$ and $X_i^t$, where $i=1,2,\cdots,d$, we have (a) $X_j^t$ is a non-descendant node of $X_i^t$ iff $X_i^{\tau} \indep X_j^t \mid \mathbf{an}_i^{\tau}$, and (b) $X_j^t$ is a descendant node of $X_i^t$ iff $X_i^{\tau} \notindep X_j^t \mid \mathbf{an}_i^{\tau}$. 
\end{theorem}

\begin{proof}
    (a) From the Def. \ref{def:ALE} of time series, we first obtain a causal path: $X_i^{\tau} \dashrightarrow X_i^t$ ($X_i^{\tau} \rightarrow X_i^{\tau+1:t-1} \rightarrow X_i^t$). The Def. \ref{def:ASCG} shows that there is not a causal path from $X_i^{\tau}$ to $\mathbf{an}_i^t$. If $X_j^t \in \mathbf{an}_i^t$, then there are only two summary path between $X_i^{\tau}$ and $X_j^t$: $X_i^{\tau} \dashleftarrow \mathbf{an}_i^{\tau}  \dashrightarrow X_j^t$ and $X_i^{\tau} \dashrightarrow \{X_i^t, \mathbf{de}_i^{t} \} \dashleftarrow  X_j^t$. Hence, once we control the conditional set $\mathbf{an}_i^{\tau}$, i.e., cut off all backdoor path, then the confounding effect between $X_i^{\tau}$ and $X_j^t$ would be eliminated and $X_i^{\tau} \indep X_j^t \mid \mathbf{an}_i^{\tau}$. Similarity, If $X_j^t \in \mathbf{sib}_i^t$, then the summary backdoor path is $X_i^{\tau} \dashleftarrow \mathbf{an}_i^{\tau}  \dashrightarrow \mathbf{an}_j^t  \dashrightarrow X_j^t$. In summary, if $X_j^t$ is a non-descendant node of $X_i^t$, then $X_i^{\tau} \indep X_j^t \mid \mathbf{an}_i^{\tau}$. In turn, given condition $X_i^{\tau} \indep X_j^t \mid \mathbf{an}_i^{\tau}$, we assume $X_j^t \in \mathbf{de}_i^t$, then we can obtain $X_i^t \dashrightarrow X_j^t$. According Def. \ref{def:CTT}, we will observe a path $X_i^{\tau} \dashrightarrow X_j^{\tau}  \dashrightarrow X_j^t$. Obviously, $X_i^{\tau} \notindep X_j^t \mid \mathbf{an}_i^{\tau}$, which contradicts the initial condition. $X_j^t$ is a non-descendant node of $X_i^t$. (b) If $X_j^t$ is a descendant node of $X_i^t$, we would have $X_i^{\tau} \dashrightarrow X_i^{t}  \dashrightarrow X_j^t$, so $X_i^{\tau} \notindep X_j^t \mid \mathbf{an}_i^{\tau}$. In turn, given the condition $X_i^{\tau} \notindep X_j^t \mid \mathbf{an}_i^{\tau}$, result (a) shows that $X_j^t$ is not a non-descendant node of $X_i^t$. Thus,  $X_j^t$ is a descendant node of $X_i^t$.
\end{proof}

In practical, under unknown causal graph, we can not directly identify ancestor nodes $\mathbf{pa}_{i}^{\tau}$.
Through a simple independence test, so we first select a set of variables $\boldsymbol{X}_{\otimes i}^{\tau}$ includes all variables at time $\tau$, except for $X_i^{\tau}$ and any variables that are independent of $X_i^{\tau}$. This means that each variable in $\boldsymbol{X}_{\otimes i}^{\tau}$ is dependent on $X_i^{\tau}$, i.e., $X_i^{\tau} \not \perp X_j^{\tau}$ for each variable $X_j^{\tau} \in \boldsymbol{X}_{\otimes i}^{\tau}$. 
As $\boldsymbol{X}_{\otimes i}^{\tau}$ occurs prior to time $t$, it does not introduce any additional backdoor paths to non-descendant nodes at time $t$, nor can it block the path $X_i^{\tau} \dashrightarrow X_i^{t} \dashrightarrow X_j^t$. 
Thus, we can directly modify the conditional set in Theorem \ref{th:our} to $\boldsymbol{X}_{\otimes i}^{\tau}$.

\begin{corollary} \label{lemma:our}
Given a two time-slice observations $\mathcal{D} = \{ \boldsymbol{X}^{\tau}, \boldsymbol{X}^{t} \}_{\tau < t}$, for the variables $X_i^{\tau}$ and $X_i^t$, where $i=1,2,\cdots,d$, we have (a) $X_j^t$ is a non-descendant node of $X_i^t$ iff $X_i^{\tau} \indep X_j^t \mid \boldsymbol{X}_{\otimes i}^{\tau}$, and (b) $X_j^t$ is a descendant node of $X_i^t$ iff $X_i^{\tau} \notindep X_j^t \mid \boldsymbol{X}_{\otimes i}^{\tau}$. 
\end{corollary}

Based on the corollary \ref{lemma:our}, then, we can distinguish between descendant and non-descendant nodes of each variable $X_i$ by a single conditional independence test per variable ($X_i^{\tau} \not \perp \mathbf{de}_i^t \mid \boldsymbol{X}_{\otimes i}^{\tau}$). Interesting, an optional random intervention on $X_i^{\tau}$ can be integrated directly into this corollary. Once we use a random intervention to replace the previous state values of some variables, the conditional set in corollary \ref{lemma:our} will become an empty set because the random intervention is independent of the other variables, and the conditional independence test in corollary \ref{lemma:our} can be replaced by a simple independence test, which will effectively accelerate the search for hierarchical topological ordering.

\subsection{HT-CIT algorithm}
\label{sec:alg}

% In this section, we present a novel approach for accurate and efficient topological ordering using the independence of conditional instrumental variables. Our approach, called the descendant-oriented topological ordering with conditional instrumental variables (HT-CIT), utilizes a simple conditional independence test to distinguish between ancestor and descendant nodes for each variable. Then, we construct the underlying DAG by adhering to the topological ordering and eliminating unnecessary edges. Unlike the recent method SCORE~\cite{rolland2022score}, which uses machine learning to approximate the Hessian of the data log-likelihood and determine the topological ordering by identifying which elements of the Hessian's diagonal are constant, our approach is simpler and more straightforward, relying solely on a conditional independence test for conditional instrumental variables. 
% For the conditional instrumental variables described above, we calculate the conditional independence using the conditional HSIC test from \cite{DBLP:conf/uai/ZhangPJS11} with Gaussian kernel. However, performing the conditional independence test is a tricky challenge, which may lead to incorrect estimations. To mitigate this issue, we introduce a descendant-orient layer ordering as a double assurance for the acyclic constraints in causal discovery. 
% Then, we prune the unnecessary edges in the adjusted topological ordering using a well-defined pruning method \cite{buhlmann2014cam, Lachapelle2020Gradient-Based}. 

\subsubsection{Identifying hierarchical topological ordering}
\label{sec:HT-CIT}

From the identification results from corollary \ref{lemma:our}, suppose we are given two time-slices $\mathcal{D} = \{ \boldsymbol{X}^{\tau}, \boldsymbol{X}^{t} \}_{\tau < t}$, then we can construct the conditional set $\boldsymbol{X}_{\otimes i}^{\tau}$ via a simple independence test $\boldsymbol{X}_{\otimes i}^{\tau} = \{ X_j^{\tau} \mid X_j^{\tau} \perp X_i^{\tau} \}$. Given $\boldsymbol{X}_{\otimes i}^{\tau}$, we can distinguish between descendant and non-descendant nodes of each variable $X_i$ by a single conditional independence test per variable ($X_i^{\tau} \not \perp \mathbf{de}_i^t \mid \boldsymbol{X}_{\otimes i}^{\tau}$). 
For every $i, j \in \{1, 2, \cdots, d\}$, we calculate the conditional independence significance $\boldsymbol{P}$ using the conditional HSIC test from \cite{DBLP:conf/uai/ZhangPJS11} with Gaussian kernel, and determine that $X_i$ is a descendant of $X_j$ if the reported $p$-value is less than or equal to a threshold $\alpha$ (i.e. if $X_i^{\tau} \not \perp X_j^t \mid \boldsymbol{X}_{\otimes i}^{\tau}$, then $X_i \rightarrow X_j$). Then, we can obtain the adjacency matrix of the unique hierarchical  topological ordering by:
\begin{eqnarray}
\label{eq:condtionalmatrix8}
\boldsymbol{P} = 
\begin{pmatrix}
p_{1,1} & p_{1,2} & \cdots & p_{1,d} \\
p_{2,1} & p_{2,2} & \cdots & p_{2,d} \\
\vdots & \vdots & \ddots & \vdots    \\
p_{d,1} & p_{d,2} & \cdots & p_{d,d} 
\end{pmatrix},~ 
\boldsymbol{A}^{TP} = 
\begin{pmatrix}
\mathbb{I}({p_{1,1} \leq \alpha}) & \mathbb{I}({p_{1,2} \leq \alpha}) & \cdots & \mathbb{I}({p_{1,d} \leq \alpha}) \\
\mathbb{I}({p_{2,1} \leq \alpha}) & \mathbb{I}({ p_{2,2} \leq \alpha}) & \cdots & \mathbb{I}({p_{2,d} \leq \alpha}) \\
\vdots & \vdots & \ddots & \vdots    \\
\mathbb{I}({p_{d,1} \leq \alpha}) & \mathbb{I}({p_{d,2} \leq \alpha}) & \cdots & \mathbb{I}({p_{d,d} \leq \alpha}) 
\end{pmatrix}.
\end{eqnarray} 
where $p_{i,j} = \mathbf{HSIC}(X_i^{\tau}, X_j^t \mid \boldsymbol{X}_{\otimes i}^{\tau})$, $\alpha$ is a hyper-parameter denoting significance threshold, and $\mathbb{I}(\cdot)$ is the indicator function. If the $p$-value is less than $\alpha$, the result is considered significant and an edge is added in the hierarchical topological ordering.
In statistical hypothesis testing, $\alpha$ is typically set to 0.05 or 0.01. In this paper, we set the hyper-parameter $\alpha=0.01$ as the default.

Despite the significant advancements in the development of conditional independence \cite{DBLP:conf/uai/ZhangPJS11,runge2018conditional, bellot2019conditional,runge2019detecting}, testing for conditional independence remains a challenging task, particularly in high-dimensional variables. The conditional independence test (HSIC) may yield imprecise decisions in implementation. Failure to accurately identify the topological ordering may result in cycles in the graph, leading to biased causal discovery. To alleviate this issue, we propose topological layer adjustment as a double guarantee for acyclic constraints in causal discovery, providing an additional safeguard against incorrect estimations in conditional independence testing.

\subsubsection{Adjusting the topological ordering}
\label{sec:layerording}
To ensure the accuracy of the conditional independence test and prevent the presence of cycles in the topological ordering,  we propose topological layer adjustment to rectify the cycle graph in ordering.

\textbf{Finding leaf nodes in the bottom layer of the topological ordering}. 
We iteratively identify the leaf nodes in the bottom layer of hierarchical topological ordering, following the principle that leaf nodes must have no descendant nodes. Therefore, if $X_i^t$ is a leaf node at the current time, then $\boldsymbol{X}_{M_{i}} \perp X_i^t \mid \boldsymbol{X}_{\otimes i}^{\tau}$ holds true, where $\boldsymbol{X}_{M_{i}} = \{\boldsymbol{X}^{\tau}/X_i^{\tau}\}$ represents all variables at time $\tau$ except for $X_i^{\tau}$.
Then we set $k=1$, then all leaf nodes would be placed into the set of layer $\boldsymbol{L}_k$, i.e., if $\boldsymbol{X}_{M_{i}} \perp X_i^t \mid \boldsymbol{X}_{\otimes i}^{\tau}$, then $X_i^t \in \boldsymbol{L}_k$.

% \begin{definition}[Hierarchical Topological Ordering]
% \label{def:HTO}
% In the hierarchical topological ordering e.g., $\Pi(\boldsymbol{X}) = (\{X_{\pi_1}\}_{\boldsymbol{L}_1}, \{X_{\pi_2}, X_{\pi_3}\}_{\boldsymbol{L}_2}, \{X_{\pi_4}, X_{\pi_5}\}_{\boldsymbol{L}_3}, \cdots) $, variables at the same layer are grouped together. Each layer is denoted by $\boldsymbol{L}_i$ and we represent $l_j$ as the located layer of $X_j$. If there is a directed edge from $X_{\pi_i}$ to $X_{\pi_j}$, then $X_{\pi_i}$ is located in a lower layer than $X_{\pi_j}$, i.e., $l_{\pi_i} < l_{\pi_j}$.
% \end{definition}
% \vspace{-4pt}

By repeating this operation, we can iteratively make $k := k + 1$ and identify the leaf nodes in the current bottom layer $\boldsymbol{L}_k$: 
\begin{eqnarray}
    X_i^t \in \boldsymbol{L}_k, \text{ if } a_{i,j}^{TP}=0 \text{ for all }j \in M_{i,k},  
\end{eqnarray} 
where $X_{M_{i,k}} = \{\boldsymbol{X}^{\tau}/X_i^{\tau}, \boldsymbol{L}_{1:k-1}\}$ denotes all variables at time $\tau$, except for $X_i^{\tau}$ and the variables in lower layer $\boldsymbol{L}_{1:k-1}$. Then $M_{i,k}$ is the index of variables $\{\boldsymbol{X}^{\tau}/X_i^{\tau}, \boldsymbol{L}_{1:k-1}\}$.

% By accurately testing the conditional independence, we can directly identify the leaf nodes from the adjacency matrix of the topological ordering:
% \begin{eqnarray}
% \label{eq:leaf1}
% X_i \text{ is a leaf node}, \text{ if } a_{i,j}^{TP}=0, j=1,2,\cdots,d,  
% \end{eqnarray} 
% which means: 
% \begin{eqnarray}
% \label{eq:leaf2}
% X_i \text{ is a leaf node}, \text{ if } p_{i,j} > \alpha, j=1,2,\cdots,d.  
% \end{eqnarray} 

\textbf{Ensuring acyclic constraints}.
An error can occur due to the difficulty of performing the conditional independence test in high-dimension variables, and the HSIC test results can sometimes be inaccurate. This can result in cycles in the topological ordering, which makes it impossible to identify any leaf node at the current topological graph since all nodes have at least one descendant node. To ensure acyclic constraints and rectify the edges in topological ordering, if the causal relationship between the unprocessed nodes in topological ordering is a directed cyclic graph, then we locate the maximum $p$-value that is less than $\alpha$ and reassign it to a value of $2\alpha$ and delete this edge in topological ordering.
\begin{eqnarray}
\label{eq:while}
p_{i^*, j^*} := 2\alpha \quad \text{and} \quad a_{i^*, j^*}^{TP}=0, \quad (i^*, j^*) = \arg \max_{i,j} (p_{i,j} \leq \alpha),  
\end{eqnarray} 
we repeat this operation until a new leaf node is identified. By adjusting the $p$-value, the layer sorting leads to a more precise hierarchical topological ordering $\boldsymbol{A}^{TP} = \{a_{i,j}^{TP}\}_{d \times d}$. This ensures that the topological ordering of the graph is acyclic and improves the accuracy of topological ordering.

% \textbf{Building layer ordering}.
% By iteratively finding and ordering leaf nodes in the descendant-oriented topological ordering, we can establish a layered ordering. Once a leaf node is found, it is not taken into consideration in subsequent layers of identifying leaf nodes.
% \begin{eqnarray}
% \label{eq:tp13}
% \Pi(\boldsymbol{X}) & = & \{ \{X_{*,l=1}\}, \{X_{*,l=2}\}, \cdots, \{X_{*,l=L} \}\} \nonumber \\
% & = & \{X_{\pi_1,l=1}, X_{\pi_2,l=1},\cdots, X_{\pi_d,l=L} \}. 
% \end{eqnarray} 
% In the layer ordering, nodes in lower layers can only be descendants of nodes in previous layers.

\subsubsection{Pruning spurious edges}
\label{sec:prune}
\vspace{-2pt}

In line with topology-based work \cite{rolland2022score, sanchez2022diffusion}, once a topological ordering is estimated, the underlying DAG is a sub-graph of the topological graph. It is necessary to further prune incorrect edges for the true DAG. Theoretically, the conditional independence between the hierarchical topological layer ordering allows for a pruning process that only requires one higher layer's nodes, current layers' nodes and two lower layers' nodes as the conditional set, or the node's non-descendants and one lower layer's nodes as the conditional set, to examine if a spurious edge exists between a node and its descendant nodes. However, classical methods such as CAM appears to perform better in practice \cite{buhlmann2014cam}, which use significance testing based on generalized additive models and select cause if the $p$-values are less than or equal to 0.001. Like \cite{rolland2022score}, we use the CAM pruning algorithm for every baseline model to prune the spurious edges. The full pseudo-code are placed in Algorithm \ref{algorithm} in Appendix \ref{app:pesudo}.

\vspace{-2pt}
\section{Numerical experiments}
\label{sec:experiments}
\vspace{-2pt}

\subsection{Baselines and evaluation}
\label{sec:baseline}
\label{sec:metrics}
\vspace{-2pt}

In this paper, we focus on two time-slices for learning directed Acyclic graphs and concatenate two time-slices data with a known temporal edge to a single concatenate dataset with a single summary causal graph. Then we apply the proposed algorithm (\textbf{HT-CIT}) to both synthetic and real-world data and compare its performance to the following baselines: 
constraint-based methods, \textbf{PC} and \textbf{FCI}~\cite{spirtes2000causation};
score-based methods, \textbf{GES}~\cite{chickering2002optimal};
continuous-optimization, \textbf{GraNDAG}~\cite{Lachapelle2020Gradient-Based}, \textbf{GOLEM}~\cite{ng2020golem}, \textbf{NOTEARS} with MLP~\cite{zheng2020learning}, and \textbf{ReScore}~\cite{zhang2023boosting}; time-seires method, \textbf{CD-NOD} \cite{huang2020causal}; topology-based methods, \textbf{CAM}~\cite{buhlmann2014cam} and \textbf{SCORE}~\cite{rolland2022score}. Besides, once a random intervention is implemented to previous time-slice, then we can use a simple independence test (IT), i.e., \textbf{HT-IT}, to replace the costly conditional independence in the proposed \textbf{HT-CIT}. 

To evaluate the performance of the proposed \textbf{HT-CIT}, we compute the Structural Hamming Distance (\textbf{SHD}) between the output and the true DAG to evaluate the differences in terms of the number of nodes, edges, and connections present in two graphs. Besides, we use Structural Intervention Distance (\textbf{SID}) to counts the minimum number of interventions needed to transform the output DAG into the true DAG, or vice versa. The accuracy of the identified edges can also be evaluated through the use of commonly adopted metrics \textbf{F1-Score} and L2-distance (\textbf{Dis.}) between two graphs.
 
In the two process of topology-based methods, SCORE~\cite{rolland2022score} typically produce a complete topological ordering with $d(d-1)/2$ edges, many of which point to non-descendants unnecessarily that must be pruned. The proposed HT-ICT use a single conditional independence test per variable to build a more efficient unique hierarchical topological ordering with merely a few spurious edges. If the number of pruned edges in the topological ordering is smaller, it can significantly improve both the efficiency and accuracy of leanred DAG. As a comparison among topology-based methods we count the number of spurious edges that needed to be pruned for each method, which is represented by \textbf{\#Prune}. 

% Metrics \textbf{SHD$\downarrow$}, \textbf{SID$\downarrow$}, \textbf{Dis.$\downarrow$} and \textbf{\#Prune$\downarrow$} are lower is better, while metrics  \textbf{F1-Score$\uparrow$} is higher is better.

\vspace{-2pt}
\subsection{Experiments on synthetic data}
\label{sec:complex}
\vspace{-2pt}

% \subsubsection{Datasets} 
% \label{sec:syn}

\textbf{Datasets}. We test our algorithm on synthetic data generated from a \emph{additive non-linear noise model} (Eq.~\ref{eq:data}) with Defs. \ref{def:ALE}, \ref{def:CTT} and \ref{def:ASCG}. For a fixed number of nodes $d$ and edges $e$, we generate the causal graph, represented by a DAG $\mathcal{G}$, using the Erdos-Renyi model \cite{erdos2011evolution}. In main experiments, we generate the data with Gaussian Noise for every variable $X_i^{\tau}$, $i=1,2,\cdots, d$ at time $\tau = 1,2,\cdots, t$:
\begin{eqnarray}
\label{eq:dataEX14}
X_i^{\tau}=\text{Sin}\left(\mathbf{pa}_{i}^{\tau}\right)+\text{Sin}\left(X_i^{\tau-1}\right)+\epsilon_i^{\tau}, \quad \boldsymbol{X}^0 \sim \mathcal{N}\left(0, \mathrm{I}_{d}\right), \quad \boldsymbol{\epsilon}^{\tau} \sim  \mathcal{N}\left(0, 0.4 \cdot \mathrm{I}_{d}\right),
% \nonumber 
% X_i^0 \sim \mathcal{N}\left(0, 1\right), \quad 
% \boldsymbol{\epsilon}^{\tau} \sim \mathcal{N}\left(0, \Sigma^{\tau}\right), 
\end{eqnarray} 
where $\text{Sin} (\mathbf{pa}_{i}^{\tau}) = \sum_{j \in \mathrm{pa}(X_i)} \sin (X_j ^{\tau})$, and $\mathrm{I}_{d}$ is a $d$ order identity matrix. In this scenario, we set $\boldsymbol{X}^{0} \sim \mathcal{N}\left(0, \mathrm{I}_{d}\right) $ to simulate the random intervention in time-slice $\tau = 0$.

To evaluate our HT-CIT on a wide range of scenarios, we vary the number of nodes ($d$) and edges ($e$) of the sampled graph and use Sin-$d$-$e$ to denote the synthetic dataset with $d$ nodes and $e$ edges. Moreover, to test the robustness of the algorithm against different noise type, we also generate data with Laplace noise ($X_i^0 \sim \operatorname{Laplace}(0,1), \boldsymbol{\epsilon}^\tau \sim \operatorname{Laplace}(0,1 / \sqrt{2})$) and Uniform noise ($X_i^0 \sim U(-1,1),  \epsilon^\tau \sim U(-1,1)$).  In each experiment setting, we perform 10 replications, each with sample size 1000, to report the mean and the standard deviation of metrics mentioned in Sec.~\ref{sec:metrics}. 

Additional, experiments on exploring \textbf{complex non-linear relationships} and on \textbf{large graph with 50/100-dimension variables} are deferred to Appendix \ref{app:nonliearn} and \ref{app:high-dimension}.

% \subsubsection{Results} 
% \label{sec:main}

\textbf{Exploring the influence of underlying DAG's size and sparsity on varying Sin-$d$-$e$ experiments}.  The results of the synthetic experiments are shown in Tables \ref{tab:Table1} and \ref{tab:Table2}. From the results on sparser graphs (Sin-10-10 and Sin-20-20) in Table \ref{tab:Table1}, we have the following observation: 
(1) In non-linear time-series data, given a concatenation of two (non-)consecutive time-slices, causal sufficient may be violated and the time dependency is complex, resulting in that the time-series variants of PC, FCI, GES and GraNDAG fail to accurately identify causal graphs.  
(2) One promising time-series method is CD-NOD, which perform PC algorithm for causal discovery on the augmented data set with time label that captures the unobserved changing factors. However, it only provides an equivalence class of the causal graph, hindering the exploration of true causality. 
(3) The three methods (GOLEM, NOTEARS and ReScore) designed specifically for sparse graphs perform well on Sin-10-10 and Sin-20-20, even exceeding the proposed HT-CIT on Sin-10-10 with pure observational data ($\mathcal{D} = \{\boldsymbol{X}^1, \boldsymbol{X}^2\}$). 
(4) As topology-based methods, CAM and SCORE achieve highly accurate causal graph. Comparing the performance on two different data types, SCORE recovered almost full causal graphs on interventional data ($\mathcal{D} = \{\boldsymbol{X}^0, \boldsymbol{X}^1\}$), but there is a drop in performance on observational data ($\mathcal{D} = \{\boldsymbol{X}^1, \boldsymbol{X}^2\}$). Because observational data contains intricate causal relationships present in previous state. 
(5) The proposed HT-CIT build a hierarchical topological ordering with merely a few spurious edges. The search space over the learned hierarchical topological ordering is much smaller than that of SCORE. On average, compared to SCORE, the number of pruned edges in HT-CIT decreases 24.4 for Sin-10-10 and 147.6 for Sin-20-20.
As the underlying DAG's size increases, HT-CIT achieves unbiased causal discovery on interventional data, but there may be a slight decrease on observational data, i.e., merely an error edge on average, and F1-Score still exceed 95\%. 

In the experiments on denser graph with more edges ($e=2d$ and $e=3d$), we selectively report on a few of the best performing baselines on observational data ($\mathcal{D} = \{\boldsymbol{X}^1, \boldsymbol{X}^2\}$) in Table \ref{tab:Table2}. Most previous baselines were only applicable to sparse graphs, whereas our algorithm exhibits substantial improvements on dense graphs. In comparison to the best baseline, our algorithm boasts a 48\% increase in SHD, a 48\% increase in SID, and a 15\% boost in F1-Score on Sin-10-20, and boasts a 30\% increase in SHD, a 43\% increase in SID, and a 7\% boost in F1-Score on Sin-10-30.

%%%%%%%%%%%%%%%%%%%%%%%%%%%%%%%%%%%%%%%%%%%%%
%%%%%%%%%%%%%%%%%%%%%%%%%%%%%%%%%%%%%%%%%%%%

\begin{table}
  \caption{The results (mean$_{\pm std}$ ) on sparse graph Sin-$d$-$e$ with simulated interventional data ( $\mathcal{D} = \{\boldsymbol{X}^0, \boldsymbol{X}^1\}$ ) or pure observational data ( $\mathcal{D} = \{\boldsymbol{X}^1, \boldsymbol{X}^2\}$ ). }
  \label{tab:Table1}
  \centering
  \resizebox{\linewidth}{!}{
  \begin{tabular}{c|ccccc|ccccc}
    \toprule
    & \multicolumn{5}{c|}{\bf Sin-10-10 Graph with Interventional Data ( $\mathcal{D} = \{\boldsymbol{X}^0, \boldsymbol{X}^1\}$ )} & \multicolumn{5}{c}{\bf Sin-10-10 Graph with Observational Data ( $\mathcal{D} = \{\boldsymbol{X}^1, \boldsymbol{X}^2\}$ )} \\
    \midrule
    Method & \textbf{SHD$\downarrow$} & \textbf{SID$\downarrow$} & \textbf{F1-Score$\uparrow$} & \textbf{Dis.$\downarrow$} & \textbf{\#Prune$\downarrow$}  & \textbf{SHD$\downarrow$} &  \textbf{SID$\downarrow$}  & \textbf{F1-Score$\uparrow$} & \textbf{Dis.$\downarrow$} & \textbf{\#Prune$\downarrow$}  \\
    \midrule
    PC & 5.90$_{\pm 3.28}$ & 34.7$_{\pm 20.8}$ & 0.77$_{\pm 0.11}$ & 2.32$_{\pm 0.74}$ & - & 12.8$_{\pm 5.03}$ & 43.6$_{\pm 9.94}$ & 0.56$_{\pm 0.12}$ & 3.51$_{\pm 0.72}$ & - \\
    FCI & 9.70$_{\pm 2.87}$ & 58.9$_{\pm 17.3}$ & 0.67$_{\pm 0.07}$ & 3.08$_{\pm 0.48}$ & - & 15.3$_{\pm 3.77}$ & 71.0$_{\pm 11.5}$ & 0.54$_{\pm 0.09}$ & 3.89$_{\pm 0.46}$ & - \\
    GES & 8.60$_{\pm 4.97}$ & 34.8$_{\pm 19.4}$ & 0.65$_{\pm 0.20}$ & 2.81$_{\pm 0.89}$ & - & 12.3$_{\pm 6.83}$ & 41.5$_{\pm 20.1}$ & 0.61$_{\pm 0.19}$ & 3.37$_{\pm 1.01}$ & - \\
    % NOTEARS & 10.5$_{\pm 1.08}$ & 36.6$_{\pm 11.4}$ & 0.14$_{\pm 0.14}$ & 3.24$_{\pm 0.17}$ & - & 10.9$_{\pm 1.29}$ & 32.1$_{\pm 6.60}$ & 0.26$_{\pm 0.14}$ & 3.30$_{\pm 0.20}$ & - \\
    CD-NOD & 3.00$_{\pm 3.16}$ & 11.3$_{\pm 13.3}$ & 0.86$_{\pm 0.14}$ & 1.28$_{\pm 1.16}$ & - & 5.40$_{\pm 0.92}$ & 15.5$_{\pm 4.70}$ & 0.74$_{\pm 0.04}$ & 2.32$_{\pm 0.19}$ & - \\
    GraNDAG & 7.80$_{\pm 2.57}$ & 25.9$_{\pm 8.13}$ & 0.60$_{\pm 0.14}$ & 2.76$_{\pm 0.47}$ & - & 19.0$_{\pm 3.74}$ & 56.1$_{\pm 4.33}$ & 0.40$_{\pm 0.10}$ & 4.34$_{\pm 0.43}$ & - \\
    GOLEM & \bf 0.00$_{\pm 0.00}$ & \bf 0.00$_{\pm 0.00}$ & \bf 1.00$_{\pm 0.00}$ & \bf 0.00$_{\pm 0.00}$ & - & \bf 0.50$_{\pm 0.80}$ & 1.80$_{\pm 2.70}$ & \bf 0.97$_{\pm 0.03}$ & \bf 0.38$_{\pm 0.59}$ & - \\
    NOTEARS & \bf 0.00$_{\pm 0.00}$ & \bf 0.00$_{\pm 0.00}$ & \bf 1.00$_{\pm 0.00}$ & \bf 0.00$_{\pm 0.00}$ & - & 1.20$_{\pm 0.60}$ & 2.30$_{\pm 1.20}$ & 0.94$_{\pm 0.02}$ & 1.02$_{\pm 0.30}$ & - \\
    ReScore & \bf 0.00$_{\pm 0.00}$ & \bf 0.00$_{\pm 0.00}$ & \bf 1.00$_{\pm 0.00}$ & \bf 0.00$_{\pm 0.00}$ & - & 1.00$_{\pm 0.63}$ & \bf 1.40$_{\pm 1.36}$ & 0.95$_{\pm 0.03}$ & 0.88$_{\pm 0.47}$ & - \\
    CAM & 5.00$_{\pm 6.27}$ & 14.9$_{\pm 18.5}$ & 0.78$_{\pm 0.27}$ & 1.53$_{\pm 1.72}$ & 80.00$_{\pm 0.00}$ & 3.70$_{\pm 2.95}$ & 13.2$_{\pm 10.6}$ & 0.84$_{\pm 0.13}$ & 1.79$_{\pm 0.74}$ & 80.00$_{\pm 0.00}$ \\
    SCORE & 1.20$_{\pm 3.46}$ & 4.2$_{\pm 10.7}$ & 0.95$_{\pm 0.14}$ & 0.43$_{\pm 1.06}$ & 35.30$_{\pm 0.95}$ & 5.60$_{\pm 3.92}$ & 21.2$_{\pm 16.1}$ & 0.78$_{\pm 0.14}$ & 2.25$_{\pm 0.78}$ & 35.80$_{\pm 0.98}$ \\
    \midrule
    \textbf{HT-CIT} & \bf 0.00$_{\pm 0.00}$ & \bf 0.00$_{\pm 0.00}$ & \bf 1.00$_{\pm 0.00}$ & \bf 0.00$_{\pm 0.00}$ & \bf 9.00$_{\pm 2.65}$ & 1.00$_{\pm 1.22}$ & 3.20$_{\pm 3.70}$ & 0.95$_{\pm 0.05}$ & 0.68$_{\pm 0.72}$ & \bf 13.20$_{\pm 4.30}$ \\
    \bottomrule
    \toprule
    Method & \multicolumn{5}{c|}{\bf Sin-20-20 Graph with Interventional Data ( $\mathcal{D} = \{\boldsymbol{X}^0, \boldsymbol{X}^1\}$ )} & \multicolumn{5}{c}{\bf Sin-20-20 Graph with Observational Data ( $\mathcal{D} = \{\boldsymbol{X}^1, \boldsymbol{X}^2\}$ )} \\
    \midrule
    PC & 10.7$_{\pm 5.70}$ & 61.2$_{\pm 35.6}$ & 0.79$_{\pm 0.10}$ & 3.18$_{\pm 0.83}$ & - & 21.5$_{\pm 6.75}$ & 98.2$_{\pm 31.8}$ & 0.61$_{\pm 0.11}$ & 4.59$_{\pm 0.69}$ & - \\
    FCI & 20.1$_{\pm 3.03}$ & 181.$_{\pm 49.9}$ & 0.66$_{\pm 0.05}$ & 4.47$_{\pm 0.35}$ & - & 30.5$_{\pm 4.09}$ & 237.$_{\pm 59.1}$ & 0.54$_{\pm 0.05}$ & 5.51$_{\pm 0.37}$ & - \\
    GES & 9.40$_{\pm 3.06}$ & 53.6$_{\pm 25.1}$ & 0.80$_{\pm 0.06}$ & 3.03$_{\pm 0.50}$ & - & 17.3$_{\pm 5.23}$ & 73.1$_{\pm 35.7}$ & 0.68$_{\pm 0.08}$ & 4.11$_{\pm 0.69}$ & - \\
    % NOTEARS & 26.6$_{\pm 12.6}$ & 117.$_{\pm 54.4}$ & 0.15$_{\pm 0.09}$ & 5.04$_{\pm 1.06}$ & - & 33.4$_{\pm 17.9}$ & 130.$_{\pm 67.1}$ & 0.12$_{\pm 0.09}$ & 5.59$_{\pm 1.38}$ & - \\
    CD-NOD & exceed 48h & - & - & - & - & - & - & - & - & - \\
    GraNDAG & 17.9$_{\pm 5.04}$ & 62.8$_{\pm 36.3}$ & 0.55$_{\pm 0.11}$ & 4.20$_{\pm 0.57}$ & - & 40.6$_{\pm 7.89}$ & 190.$_{\pm 46.2}$ & 0.38$_{\pm 0.08}$ & 6.34$_{\pm 0.63}$ & - \\
    GOLEM & 0.60$_{\pm 1.50}$ & 2.50$_{\pm 5.20}$ & 0.98$_{\pm 0.04}$ & 0.32$_{\pm 0.70}$ & - & 1.30$_{\pm 1.10}$ & 5.60$_{\pm 4.40}$ & 0.97$_{\pm 0.03}$ & 0.93$_{\pm 0.66}$ & - \\
    NOTEARS & 0.20$_{\pm 0.40}$ & 1.00$_{\pm 2.0}$ & 0.99$_{\pm 0.01}$ & 0.20$_{\pm 0.40}$ & - & 2.60$_{\pm 1.49}$ & 6.00$_{\pm 3.40}$ & 0.94$_{\pm 0.03}$ & 1.55$_{\pm 0.46}$ & - \\
    ReScore & 0.90$_{\pm 2.70}$ & 3.90$_{\pm 11.7}$ & 0.98$_{\pm 0.06}$ & 0.30$_{\pm 0.90}$ & - & 2.00$_{\pm 0.77}$ & 5.10$_{\pm 2.90}$ & 0.95$_{\pm 0.01}$ & 1.38$_{\pm 0.28}$ & - \\
    CAM & 4.50$_{\pm 3.03}$ & 15.8$_{\pm 14.2}$ & 0.89$_{\pm 0.07}$ & 1.86$_{\pm 1.07}$ & 360.0$_{\pm 0.00}$ & 10.3$_{\pm 6.50}$ & 41.6$_{\pm 34.7}$ & 0.79$_{\pm 0.12}$ & 3.07$_{\pm 0.98}$ & 360.0$_{\pm 0.00}$ \\
    SCORE & 0.20$_{\pm 0.63}$ & 0.90$_{\pm 2.70}$ & 0.99$_{\pm 0.02}$ & 0.14$_{\pm 0.45}$ & 170.1$_{\pm 0.32}$ & 7.40$_{\pm 2.41}$ & 31.3$_{\pm 21.7}$ & 0.85$_{\pm 0.04}$ & 2.68$_{\pm 0.47}$ & 172.1$_{\pm 0.22}$ \\
    \midrule
    \textbf{HT-CIT} & \bf 0.00$_{\pm 0.00}$ & \bf 0.00$_{\pm 0.00}$ & \bf 1.00$_{\pm 0.00}$ & \bf 0.00$_{\pm 0.00}$ & \bf 16.44$_{\pm 3.81}$ & \bf 1.00$_{\pm 1.32}$ & \bf 3.10$_{\pm 4.40}$ & \bf 0.98$_{\pm 0.03}$ & \bf 0.51$_{\pm 0.61}$ & \bf 30.60$_{\pm 7.70}$ \\
    \bottomrule
  \end{tabular}
  }
  \vspace{-8pt}
\end{table}

\begin{table}
  \caption{The results (mean$_{\pm std}$ ) on denser graph Sin-$d$-$e$ with observations ( $\mathcal{D} = \{\boldsymbol{X}^1, \boldsymbol{X}^2\}$ ). }
  \label{tab:Table2}
  \centering
  \resizebox{\linewidth}{!}{
  \begin{tabular}{c|ccccc|ccccc}
    \toprule
    & \multicolumn{5}{c|}{\bf Sin-10-20 Graph with Observational Data ( $\mathcal{D} = \{\boldsymbol{X}^1, \boldsymbol{X}^2\}$ )} & \multicolumn{5}{c}{\bf Sin-10-30 Graph with Observational Data ( $\mathcal{D} = \{\boldsymbol{X}^1, \boldsymbol{X}^2\}$ )} \\
    \midrule
    Method & \textbf{SHD$\downarrow$} & \textbf{SID$\downarrow$} & \textbf{F1-Score$\uparrow$} & \textbf{Dis.$\downarrow$} & \textbf{\#Prune$\downarrow$}  & \textbf{SHD$\downarrow$} &  \textbf{SID$\downarrow$}  & \textbf{F1-Score$\uparrow$} & \textbf{Dis.$\downarrow$} & \textbf{\#Prune$\downarrow$}  \\
    \midrule
    % CD-NOD & exceed 15h & - & - & - & - & exceed 15h & - & - & - & - \\
    GOLEM & 16.4$_{\pm 3.13}$ & 60.6$_{\pm 7.7}$ & 0.51$_{\pm 0.09}$ & 4.03$_{\pm 0.41}$ & - & 22.3$_{\pm 4.20}$ & 61.4$_{\pm 13.60}$ & 0.50$_{\pm 0.09}$ & 4.70$_{\pm 0.44}$ & - \\
    NOTEARS& 18.5$_{\pm 3.50}$ & 60.0$_{\pm 8.2}$ & 0.54$_{\pm 0.09}$ & 4.30$_{\pm 0.42}$ & - & 23.4$_{\pm 5.30}$ & 62.7$_{\pm 13.10}$ & 0.55$_{\pm 0.10}$ & 4.80$_{\pm 0.57}$ & - \\
    ReScore & 17.5$_{\pm 4.08}$ & 57.1$_{\pm 9.65}$ & 0.54$_{\pm 0.11}$ & 4.15$_{\pm 0.51}$ & - & 22.9$_{\pm 5.60}$ & 61.1$_{\pm 13.00}$ & 0.54$_{\pm 0.11}$ & 4.70$_{\pm 0.63}$ & - \\
    CAM & 9.80$_{\pm 4.76}$ & 37.1$_{\pm 11.3}$ & 0.75$_{\pm 0.12}$ & 3.03$_{\pm 0.82}$ & 70.00$_{\pm 0.00}$ & 25.6$_{\pm 5.93}$ & 60.4$_{\pm 9.480}$ & 0.59$_{\pm 0.09}$ & 5.03$_{\pm 0.60}$ & 60.00$_{\pm 0.00}$ \\
    SCORE & 16.0$_{\pm 4.92}$ & 53.6$_{\pm 10.6}$ & 0.62$_{\pm 0.11}$ & 3.95$_{\pm 0.68}$ & 31.90$_{\pm 2.02}$ & 20.3$_{\pm 7.17}$ & 51.1$_{\pm 19.14}$ & 0.68$_{\pm 0.11}$ & 4.43$_{\pm 0.86}$ & 21.70$_{\pm 2.21}$ \\
    \midrule
    \textbf{HT-CIT} & \bf 5.10$_{\pm 3.25}$ & \bf 16.1$_{\pm 6.82}$ & \bf 0.86$_{\pm 0.09}$ & \bf 2.14$_{\pm 0.77}$ & \bf 10.20$_{\pm 3.97}$ & \bf 14.1$_{\pm 3.73}$ & \bf 29.0$_{\pm 10.3}$ & \bf 0.73$_{\pm 0.07}$ & \bf 3.73$_{\pm 0.49}$ & \bf 9.40$_{\pm 3.98}$ \\
    \bottomrule
  \end{tabular}
  }
  \vspace{-14pt}
\end{table}

\begin{minipage}{\textwidth}
 \begin{minipage}[t]{0.44\textwidth}
     \centering
    \makeatletter\def\@captype{table}\makeatother\caption{Average running time(s). }
     \label{tab:Table3}
     \resizebox{\linewidth}{!}{
        \renewcommand{\arraystretch}{1.42}
        \begin{tabular}{c|cccc}
        \toprule
        & Sin-10-10 & Sin-10-20 & Sin-10-30 & Sin-20-20 \\
        \midrule
        CD-NOD & 5433s & > 1.5h & > 1.5h & >5h  \\
        GraNDAG & 343.3s & 327.4s & 447.7s & 864.7s \\
        CAM & 97.2s & 92.8s & 111.3s & 543.6s \\
        % NOTEARS & 47.7s & 57.4s & 58.7s & 88.5s \\
        \textbf{HT-CIT} & 54.7s & 77.1s & 58.2s & 224.4s \\
        GOLEM & 44.9s & 45.8s & 47.6s & 63.0s \\
        SCORE & 40.7s & 38.1s & 45.2s & 193.1s \\
        NOTEARS& 33.6s & 35.6s & 36.4s & 747.2s \\
        GES & 32.7s & 33.1s & 32.1s & 71.5s \\
        ReScore & 24.1s & 23.2s & 24.2s & \bf 29.2s \\
        PC & 21.1s & 20.7s & 21.1s & 32.8s \\
        FCI & 18.7s & 18.7s & 18.8s & 30.1s \\
        \textbf{HT-IT} & \bf 12.0s & \bf 16.8s & \bf 12.7s & 33.5s \\
        \bottomrule
        \end{tabular}
    }
  \end{minipage}
  \hfill
  \begin{minipage}[t]{0.533\textwidth}
   \centering
    \makeatletter\def\@captype{table}\makeatother\caption{The experiments on different noise type. }
    \label{tab:Table4}
        \resizebox{\linewidth}{!}{
          \begin{tabular}{c|ccccc}
            \toprule
            & \multicolumn{5}{c}{\bf Sin-10-10 data with Laplace noise ($\mathcal{D} = \{\boldsymbol{X}^1, \boldsymbol{X}^2\}$)} \\
            \midrule
            Method & \textbf{SHD$\downarrow$} & \textbf{SID$\downarrow$} & \textbf{F1-Score$\uparrow$} & \textbf{Dis.$\downarrow$} & \textbf{\#Prune$\downarrow$}  \\
            \midrule
            GOLEM & 1.50$_{\pm 1.20}$ & \bf 2.80$_{\pm 2.52}$ & 0.92$_{\pm 0.05}$ & 1.00$_{\pm 0.70}$ & - \\
            NOTEARS& 1.60$_{\pm 0.06}$ & 3.70$_{\pm 3.10}$ & 0.92$_{\pm 0.03}$ & 1.23$_{\pm 0.26}$ & - \\
            ReScore & 2.00$_{\pm 1.34}$ & 3.00$_{\pm 2.41}$ & 0.90$_{\pm 0.06}$ & 1.29$_{\pm 0.57}$ & - \\
            CAM & 5.30$_{\pm 2.83}$ & 14.0$_{\pm 8.01}$ & 0.78$_{\pm 0.12}$ & 2.23$_{\pm 0.57}$ & 80.0$_{\pm 0.00}$ \\
            SCORE & 3.90$_{\pm 1.70}$ & 9.90$_{\pm 6.01}$ & 0.84$_{\pm 0.06}$ & 1.93$_{\pm 0.43}$ & 35.5$_{\pm 0.92}$ \\
            \midrule
            \textbf{HT-CIT} & \bf 1.20$_{\pm 1.99}$ & 3.60$_{\pm 6.55}$ & \bf 0.94$_{\pm 0.04}$ & \bf 0.59$_{\pm 0.92}$ & \bf 0.80$_{\pm 1.40}$ \\
            \bottomrule
            \toprule
            Method  & \multicolumn{5}{c}{\bf Sin-10-10 data with Uniform noise ($\mathcal{D} = \{\boldsymbol{X}^1, \boldsymbol{X}^2\}$)} \\
            \midrule
            GOLEM & 2.60$_{\pm 1.80}$ & 6.80$_{\pm 3.94}$ & 0.89$_{\pm 0.06}$ & 1.46$_{\pm 0.68}$ & - \\
            NOTEARS & 2.00$_{\pm 1.34}$ & 4.80$_{\pm 1.30}$ & 0.91$_{\pm 0.05}$ & 1.29$_{\pm 0.57}$ & - \\
            ReScore  & 1.70$_{\pm 0.90}$ & 3.70$_{\pm 2.90}$ & 0.92$_{\pm 0.04}$ & 1.21$_{\pm 0.48}$ & - \\
            CAM & 8.90$_{\pm 7.15}$ & 21.4$_{\pm 12.0}$ & 0.68$_{\pm 0.22}$ & 2.14$_{\pm 0.73}$ & 80.0$_{\pm 0.00}$ \\
            SCORE & 5.10$_{\pm 3.42}$ & 13.6$_{\pm 8.30}$ & 0.80$_{\pm 0.11}$ & 2.14$_{\pm 0.73}$ & 35.0$_{\pm 0.00}$ \\
            \midrule
            \textbf{HT-CIT} & \bf 1.00$_{\pm 2.19}$ & \bf 1.10$_{\pm 2.47}$ & \bf 0.96$_{\pm 0.09}$ & \bf 0.44$_{\pm 0.90}$ & \bf 0.70$_{\pm 1.55}$ \\
            \bottomrule
          \end{tabular}
          }     
   \end{minipage}
\end{minipage}

\begin{figure}[t]
\centerline{\includegraphics[width=0.9\columnwidth]{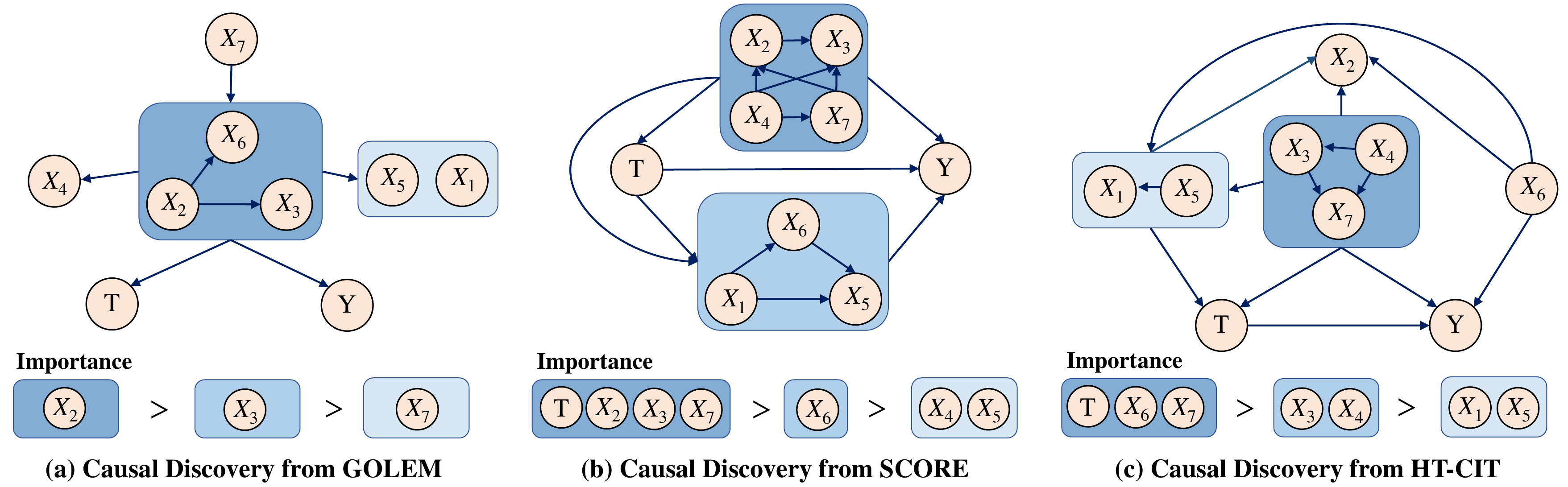}}
\caption{Causal discovery on PM-CMR dataset.}
\label{fig:figure2}
\vspace{-20pt}
\end{figure}

%%%%%%%%%%%%%%%%%%%%%%%%%%%%%%%%%%%%%%%%%%%%%
%%%%%%%%%%%%%%%%%%%%%%%%%%%%%%%%%%%%%%%%%%%%

\textbf{Scaling to different noise types}. To evaluate algorithm robustness against different noise types, Sin-10-10 data was generated with Laplace and Uniform noise. Results (Table \ref{tab:Table4}) show HT-CIT's superiority and robustness across noise types and the accuracy remains similar for Gaussian noise.

\textbf{Training cost analysis}.
In all synthetic datasets Sin-$d$-$e$, we implement 10 replications to study the average running time(s) for the proposed model in a single execution and sorted it by time spent on sin-10-10 in Table \ref{tab:Table3}. In the \textbf{HT-CIT} implementation, we use an independence test (\textbf{HT-IT}) for random intervention, which will be easier to perform and take less time than conditional independence test (\textbf{HT-CIT}) for pure observational data. From the results, the time consumption of the proposed $\textbf{HT-IT}$ is almost minimal on most of the datasets. Although conditional independence test of \textbf{HT-CIT} would increase model complexity and training cost, its single execution time is less than 300 seconds, which is still within the acceptable range. We believe \textbf{HT-CIT} is scalable to larger graphs and denser graphs, applicable to a wide range of scenarios, and its time consumption is at a low level.

\vspace{-4pt}
\subsection{Real-world data}
\label{sec:real}
\vspace{-4pt}

We apply GOLEM \cite{ng2020golem}, SCORE \cite{rolland2022score} and the proposed HT-CIT to public PM-CMR data \cite{wyatt2020dataset} (9 nodes $\{T, Y, X_{1:7}\}$, 2132 observations, the detailed description of PM-CMR is deferred to Appendix \ref{app:real}) in 2000 \& 2010, and recover a CMR-related causal graph. As illustrated in Fig.~\ref{fig:figure2} (the adjacency matrix of DAG is placed in Appendix \ref{app:real}), it can be inferred that the DAG recovered by HT-CIT contains more accurate information. In prior studies \cite{wyatt2020dataset}, $\{X_1, X_2, \cdots, X_7\}$ were thought to be confounders in the causal relation of $T$ to $Y$. However, in Figs. \ref{fig:figure2}(a-b), GOLEM shows there are no direct edge from $T$ to $Y$ and SCORE show that $T$ is the parent node of $\{X_1, X_5, X_6\}$, which contradicts previous research, which seems inconsistent with the settings in causal effect literature.

% To effectively combat cardiovascular disease, it is recommended for cities to disseminate information about its dangers, promote prevention, and provide medical care for low-income families. 
% In addition, the results show that income is influenced by all aspects of the city and does not directly affect other variables. 

\vspace{-4pt}
\section{Conclusion}
\label{sec:con}
\vspace{-4pt}

In a consistent time series data throughout time with acyclic summary causal graph, we show how two time-slice help topological ordering for learning DAGs. In such cases,  we propose HT-CIT, a novel topological sorting algorithm that utilizes conditional independence tests per node to distinguish between its descendant and non-descendant nodes and build a unique hierarchical topological ordering with a few spurious edges for identifying DAGs. One limitation is that we require two time-slices are collected over a short period of time to maintain acyclic summary graph. Besides, challenges arise when applying our algorithm to larger graphs due to the difficulty of conditional independence test.

\bibliographystyle{abbrv}
\bibliography{ref.bib}

\begin{thebibliography}{10}

\bibitem{ahammad2021new}
T.~Ahammad, M.~Hasan, and M.~Zahid~Hassan.
\newblock A new topological sorting algorithm with reduced time complexity.
\newblock In {\em Proceedings of the 3rd International Conference on
  Intelligent Computing and Optimization 2020 (ICO 2020)}, pages 418--429.
  Springer, 2021.

\bibitem{assaad2022survey}
C.~K. Assaad, E.~Devijver, and E.~Gaussier.
\newblock Survey and evaluation of causal discovery methods for time series.
\newblock {\em Journal of Artificial Intelligence Research}, 73:767--819, 2022.

\bibitem{bellot2019conditional}
A.~Bellot and M.~van~der Schaar.
\newblock Conditional independence testing using generative adversarial
  networks.
\newblock {\em Advances in Neural Information Processing Systems}, 32, 2019.

\bibitem{buhlmann2014cam}
P.~B{\"u}hlmann, J.~Peters, and J.~Ernest.
\newblock Cam: Causal additive models, high-dimensional order search and
  penalized regression.
\newblock {\em The Annals of Statistics}, 42(6):2526--2556, 2014.

\bibitem{bussmann2021neural}
B.~Bussmann, J.~Nys, and S.~Latr{\'e}.
\newblock Neural additive vector autoregression models for causal discovery in
  time series.
\newblock In {\em Discovery Science: 24th International Conference, DS 2021,
  Halifax, NS, Canada, October 11--13, 2021, Proceedings 24}, pages 446--460.
  Springer, 2021.

\bibitem{chen2021fritl}
W.~Chen, K.~Zhang, R.~Cai, B.~Huang, J.~Ramsey, Z.~Hao, and C.~Glymour.
\newblock Fritl: A hybrid method for causal discovery in the presence of latent
  confounders.
\newblock {\em arXiv preprint arXiv:2103.14238}, 2021.

\bibitem{chickering2002optimal}
D.~M. Chickering.
\newblock Optimal structure identification with greedy search.
\newblock {\em Journal of machine learning research}, 3(Nov):507--554, 2002.

\bibitem{erdos2011evolution}
P.~Erd{\"o}s and A.~R{\'e}nyi.
\newblock On the evolution of random graphs.
\newblock In {\em The Structure and Dynamics of Networks}, pages 38--82.
  Princeton University Press, 2011.

\bibitem{ghoshal2018learning}
A.~Ghoshal and J.~Honorio.
\newblock Learning linear structural equation models in polynomial time and
  sample complexity.
\newblock In {\em International Conference on Artificial Intelligence and
  Statistics}, pages 1466--1475. PMLR, 2018.

\bibitem{granger1969investigating}
C.~W. Granger.
\newblock Investigating causal relations by econometric models and
  cross-spectral methods.
\newblock {\em Econometrica: journal of the Econometric Society}, pages
  424--438, 1969.

\bibitem{hasan2023survey}
U.~Hasan, E.~Hossain, and M.~O. Gani.
\newblock A survey on causal discovery methods for temporal and non-temporal
  data.
\newblock {\em arXiv preprint arXiv:2303.15027}, 2023.

\bibitem{hauser2012characterization}
A.~Hauser and P.~B{\"u}hlmann.
\newblock Characterization and greedy learning of interventional markov
  equivalence classes of directed acyclic graphs.
\newblock {\em The Journal of Machine Learning Research}, 13(1):2409--2464,
  2012.

\bibitem{huang2020causal}
B.~Huang, K.~Zhang, J.~Zhang, J.~Ramsey, R.~Sanchez-Romero, C.~Glymour, and
  B.~Sch{\"o}lkopf.
\newblock Causal discovery from heterogeneous/nonstationary data.
\newblock {\em The Journal of Machine Learning Research}, 21(1):3482--3534,
  2020.

\bibitem{hyttinen2014constraint}
A.~Hyttinen, F.~Eberhardt, and M.~J{\"a}rvisalo.
\newblock Constraint-based causal discovery: Conflict resolution with answer
  set programming.
\newblock In {\em UAI}, pages 340--349, 2014.

\bibitem{ke2019learning}
N.~R. Ke, O.~Bilaniuk, A.~Goyal, S.~Bauer, H.~Larochelle, B.~Sch{\"o}lkopf,
  M.~C. Mozer, C.~Pal, and Y.~Bengio.
\newblock Learning neural causal models from unknown interventions.
\newblock {\em arXiv preprint arXiv:1910.01075}, 2019.

\bibitem{Lachapelle2020Gradient-Based}
S.~Lachapelle, P.~Brouillard, T.~Deleu, and S.~Lacoste-Julien.
\newblock Gradient-based neural dag learning.
\newblock In {\em International Conference on Learning Representations}, 2020.

\bibitem{li2022hybrid}
Y.~Li, R.~Xia, C.~Liu, and L.~Sun.
\newblock A hybrid causal structure learning algorithm for mixed-type data.
\newblock In {\em Proceedings of the AAAI Conference on Artificial
  Intelligence}, volume~36, pages 7435--7443, 2022.

\bibitem{loh2014high}
P.-L. Loh and P.~B{\"u}hlmann.
\newblock High-dimensional learning of linear causal networks via inverse
  covariance estimation.
\newblock {\em The Journal of Machine Learning Research}, 15(1):3065--3105,
  2014.

\bibitem{lowe2022amortized}
S.~L{\"o}we, D.~Madras, R.~Zemel, and M.~Welling.
\newblock Amortized causal discovery: Learning to infer causal graphs from
  time-series data.
\newblock In {\em Conference on Causal Learning and Reasoning}, pages 509--525.
  PMLR, 2022.

\bibitem{malinsky2018causal}
D.~Malinsky and D.~Danks.
\newblock Causal discovery algorithms: A practical guide.
\newblock {\em Philosophy Compass}, 13(1):e12470, 2018.

\bibitem{montagna2023causal}
F.~Montagna, N.~Noceti, L.~Rosasco, K.~Zhang, and F.~Locatello.
\newblock Causal discovery with score matching on additive models with
  arbitrary noise.
\newblock In {\em 2nd Conference on Causal Learning and Reasoning}, 2023.

\bibitem{nauta2019causal}
M.~Nauta, D.~Bucur, and C.~Seifert.
\newblock Causal discovery with attention-based convolutional neural networks.
\newblock {\em Machine Learning and Knowledge Extraction}, 1(1):312--340, 2019.

\bibitem{ng2020golem}
I.~Ng, A.~Ghassami, and K.~Zhang.
\newblock On the role of sparsity and dag constraints for learning linear dags.
\newblock {\em Advances in Neural Information Processing Systems},
  33:17943--17954, 2020.

\bibitem{park2017bayesian}
Y.~W. Park and D.~Klabjan.
\newblock Bayesian network learning via topological order.
\newblock {\em The Journal of Machine Learning Research}, 18(1):3451--3482,
  2017.

\bibitem{pearl2009causality}
J.~Pearl.
\newblock {\em Causality}.
\newblock Cambridge university press, 2009.

\bibitem{peters2016causal}
J.~Peters, P.~B{\"u}hlmann, and N.~Meinshausen.
\newblock Causal inference by using invariant prediction: identification and
  confidence intervals.
\newblock {\em Journal of the Royal Statistical Society: Series B (Statistical
  Methodology)}, 78(5):947--1012, 2016.

\bibitem{peters2017elements}
J.~Peters, D.~Janzing, and B.~Sch{\"o}lkopf.
\newblock {\em Elements of causal inference: foundations and learning
  algorithms}.
\newblock The MIT Press, 2017.

\bibitem{peters2014causal}
J.~Peters, J.~M. Mooij, D.~Janzing, and B.~Sch{\"o}lkopf.
\newblock Causal discovery with continuous additive noise models.
\newblock {\em The Journal of Machine Learning Research}, 15(1):2009--2053,
  2014.

\bibitem{ramsey2012adjacency}
J.~Ramsey, J.~Zhang, and P.~L. Spirtes.
\newblock Adjacency-faithfulness and conservative causal inference.
\newblock {\em arXiv preprint arXiv:1206.6843}, 2012.

\bibitem{reisach2023simple}
A.~G. Reisach, M.~Tami, C.~Seiler, A.~Chambaz, and S.~Weichwald.
\newblock Simple sorting criteria help find the causal order in additive noise
  models.
\newblock {\em arXiv preprint arXiv:2303.18211}, 2023.

\bibitem{rolland2022score}
P.~Rolland, V.~Cevher, M.~Kleindessner, C.~Russell, D.~Janzing,
  B.~Sch{\"o}lkopf, and F.~Locatello.
\newblock Score matching enables causal discovery of nonlinear additive noise
  models.
\newblock In {\em International Conference on Machine Learning}, pages
  18741--18753. PMLR, 2022.

\bibitem{runge2018conditional}
J.~Runge.
\newblock Conditional independence testing based on a nearest-neighbor
  estimator of conditional mutual information.
\newblock In {\em International Conference on Artificial Intelligence and
  Statistics}, pages 938--947. PMLR, 2018.

\bibitem{runge2019detecting}
J.~Runge, P.~Nowack, M.~Kretschmer, S.~Flaxman, and D.~Sejdinovic.
\newblock Detecting and quantifying causal associations in large nonlinear time
  series datasets.
\newblock {\em Science advances}, 5(11):eaau4996, 2019.

\bibitem{sanchez2022diffusion}
P.~Sanchez, X.~Liu, A.~Q. O'Neil, and S.~A. Tsaftaris.
\newblock Diffusion models for causal discovery via topological ordering.
\newblock {\em arXiv preprint arXiv:2210.06201}, 2022.

\bibitem{solus2021consistency}
L.~Solus, Y.~Wang, and C.~Uhler.
\newblock Consistency guarantees for greedy permutation-based causal inference
  algorithms.
\newblock {\em Biometrika}, 108(4):795--814, 2021.

\bibitem{spirtes2000causation}
P.~Spirtes, C.~N. Glymour, R.~Scheines, and D.~Heckerman.
\newblock {\em Causation, prediction, and search}.
\newblock MIT press, 2000.

\bibitem{sun2007kernel}
X.~Sun, D.~Janzing, B.~Sch{\"o}lkopf, and K.~Fukumizu.
\newblock A kernel-based causal learning algorithm.
\newblock In {\em Proceedings of the 24th international conference on Machine
  learning}, pages 855--862, 2007.

\bibitem{teyssier2005ordering}
M.~Teyssier and D.~Koller.
\newblock Ordering-based search: a simple and effective algorithm for learning
  bayesian networks.
\newblock In {\em Proceedings of the Twenty-First Conference on Uncertainty in
  Artificial Intelligence}, pages 584--590, 2005.

\bibitem{triantafillou2017predicting}
S.~Triantafillou, V.~Lagani, C.~Heinze-Deml, A.~Schmidt, J.~Tegner, and
  I.~Tsamardinos.
\newblock Predicting causal relationships from biological data: Applying
  automated causal discovery on mass cytometry data of human immune cells.
\newblock {\em Scientific reports}, 7(1):1--11, 2017.

\bibitem{tsamardinos2006max}
I.~Tsamardinos, L.~E. Brown, and C.~F. Aliferis.
\newblock The max-min hill-climbing bayesian network structure learning
  algorithm.
\newblock {\em Machine learning}, 65(1):31--78, 2006.

\bibitem{vandenbroucke2016causality}
J.~P. Vandenbroucke, A.~Broadbent, and N.~Pearce.
\newblock Causality and causal inference in epidemiology: the need for a
  pluralistic approach.
\newblock {\em International journal of epidemiology}, 45(6):1776--1786, 2016.

\bibitem{wang2017permutation}
Y.~Wang, L.~Solus, K.~Yang, and C.~Uhler.
\newblock Permutation-based causal inference algorithms with interventions.
\newblock {\em Advances in Neural Information Processing Systems}, 30, 2017.

\bibitem{wyatt2020dataset}
L.~H. Wyatt, G.~C.~L. Peterson, T.~J. Wade, L.~M. Neas, and A.~G. Rappold.
\newblock Annual pm2. 5 and cardiovascular mortality rate data: Trends modified
  by county socioeconomic status in 2,132 us counties.
\newblock {\em Data in brief}, 30:105--318, 2020.

\bibitem{yang2018characterizing}
K.~Yang, A.~Katcoff, and C.~Uhler.
\newblock Characterizing and learning equivalence classes of causal dags under
  interventions.
\newblock In {\em International Conference on Machine Learning}, pages
  5541--5550. PMLR, 2018.

\bibitem{zhang2023boosting}
A.~Zhang, F.~Liu, W.~Ma, Z.~Cai, X.~Wang, and T.-S. Chua.
\newblock Boosting causal discovery via adaptive sample reweighting.
\newblock In {\em The Eleventh International Conference on Learning
  Representations}, 2023.

\bibitem{zhang2008completeness}
J.~Zhang.
\newblock On the completeness of orientation rules for causal discovery in the
  presence of latent confounders and selection bias.
\newblock {\em Artificial Intelligence}, 172(16-17):1873--1896, 2008.

\bibitem{DBLP:conf/uai/ZhangPJS11}
K.~Zhang, J.~Peters, D.~Janzing, and B.~Sch{\"{o}}lkopf.
\newblock Kernel-based conditional independence test and application in causal
  discovery.
\newblock In {\em {UAI}}, pages 804--813. {AUAI} Press, 2011.

\bibitem{zheng2018NOTEARS}
X.~Zheng, B.~Aragam, P.~K. Ravikumar, and E.~P. Xing.
\newblock Dags with no tears: Continuous optimization for structure learning.
\newblock In S.~Bengio, H.~Wallach, H.~Larochelle, K.~Grauman, N.~Cesa-Bianchi,
  and R.~Garnett, editors, {\em Advances in Neural Information Processing
  Systems}, volume~31. Curran Associates, Inc., 2018.

\bibitem{zheng2018dags}
X.~Zheng, B.~Aragam, P.~K. Ravikumar, and E.~P. Xing.
\newblock Dags with no tears: Continuous optimization for structure learning.
\newblock {\em Advances in Neural Information Processing Systems}, 31, 2018.

\bibitem{zheng2020learning}
X.~Zheng, C.~Dan, B.~Aragam, P.~Ravikumar, and E.~Xing.
\newblock Learning sparse nonparametric dags.
\newblock In {\em International Conference on Artificial Intelligence and
  Statistics}, pages 3414--3425. PMLR, 2020.

\bibitem{zhu2020causal}
S.~Zhu, I.~Ng, and Z.~Chen.
\newblock Causal discovery with reinforcement learning.
\newblock In {\em International Conference on Learning Representations}, 2020.

\end{thebibliography}

%%%%%%%%%%%%%%%%%%%%%%%%%%%%%%%%%%%%%%%%%%%%%%%%%%%%%%%%%%%

\newpage
\appendix

\section{Pseudo-code and experiments}

\subsection{Pseudo-code}
\label{app:pesudo}

In a consistent time series data throughout time with acyclic summary causal graph, we show how two time-slice help topological ordering for learning DAGs. In such cases,  we propose HT-CIT, a novel topological sorting algorithm that utilizes conditional independence tests per node to distinguish between its descendant and non-descendant nodes and build a unique hierarchical topological ordering with a few spurious edges for identifying DAGs.  Algorithm \ref{algorithm} shows the pseudo-code of our HT-CIT. 

Hardware used: Ubuntu 16.04.3 LTS operating system with 2 * Intel Xeon E5-2660 v3 @ 2.60GHz CPU (40 CPU cores, 10 cores per physical CPU, 2 threads per core), 256 GB of RAM, and 4 * GeForce GTX TITAN X GPU with 12GB of VRAM.
    
Software used: Python 3.8 with cdt 0.6.0, ylearn 0.2.0, causal-learn 0.1.3, GPy 1.10.0, igraph 0.10.4, scikit-learn 1.2.2, networkx 2.8.5, pytorch 2.0.0.

\begin{algorithm}[h]
    \caption{HT-CIT: Hierarchical Topological Ordering with Conditional Independence Test}
    \label{algorithm}
    \begin{algorithmic}
	\STATE \textbf{Input:} Two time-slices $\mathcal{D} = \{ \boldsymbol{X}^{\tau}, \boldsymbol{X}^{t} \}_{\tau < t}$ with $d$ nodes; two significance threshold $\alpha=0.01$ and $\beta=0.001$ for conditional independence test and pruning process; the layer index $k=0$.  
	\STATE \textbf{Output:} One adjacency matrix of hierarchical topological ordering $\boldsymbol{A}^{TP}$, one directed acyclic graph $\mathcal{G}$. 
	\STATE \textbf{Components:} Conditional independence test $\mathbf{HSIC}(\dots)$; and pruning process $\mathbf{CAM}(\cdots)$. 
	\STATE \textbf{Stage 1 - Identifying Hierarchical Topological Ordering:}
	\FOR{$i=1$ {\bfseries to} $d$}
        \STATE Construct the conditional set $\boldsymbol{X}_{\otimes i}^{\tau}$ via a simple independence test $\boldsymbol{X}_{\otimes i}^{\tau} = \{ X_j^{\tau} \mid X_j^{\tau} \perp X_i^{\tau} \}$
        \FOR{$j=1$ {\bfseries to} $d$}
	       \STATE $p_{i,j} = \mathbf{HSIC}(X_i^{\tau}, X_j^t \mid \boldsymbol{X}_{\otimes i}^{\tau})$
	       \STATE $a_{i,j}^{TP} = \mathbb{I}({p_{i,j} \leq \alpha})$
        \ENDFOR
	\ENDFOR
    \STATE We obtain $\boldsymbol{P} = \{p_{i,j}\}_{d \times d}$ and $\boldsymbol{A}^{TP} = \{a_{i,j}^{TP}\}_{d \times d}$
	\STATE \textbf{Stage 2 - Adjusting the Topological Ordering:}
	\WHILE{The causal relationship between the unprocessed nodes is a directed cyclic graph}
        \STATE $k := k + 1$
        \STATE $X_{M_{i,k}} = \{X^{\tau}/X_i^{\tau}, \boldsymbol{L}_{1:k-1}\}$ 
        \STATE $X_i^t \in \boldsymbol{L}_k$, if $a_{i,j}^{TP}=0$ for all $j \in M_{i,k}$
        \WHILE{$\boldsymbol{L}_k = \emptyset $}
	      \STATE $p_{i^*, j^*} := 2\alpha \quad \text{and} \quad a_{i^*, j^*}^{TP}=0, \quad (i^*, j^*) = \arg \max_{i,j} (p_{i,j} \leq \alpha)$
            \STATE $X_i^t \in \boldsymbol{L}_k$, if $a_{i,j}^{TP}=0$ for all $j \in M_{i,k}$
        \ENDWHILE
	\STATE We obtain $\boldsymbol{P} = \{p_{i,j}\}_{d \times d}$ and $\boldsymbol{A}^{TP} = \{a_{i,j}^{TP}\}_{d \times d}$
	\ENDWHILE
    \STATE \textbf{Stage 3 - Pruning Spurious Edges:}
    \STATE We obtain $\mathcal{G} = \mathbf{CAM}(\mathcal{D}, \boldsymbol{A}^{TP}, \beta)$
    \STATE \textbf{Return: $\boldsymbol{A}^{TP}$ and $\mathcal{G}$}
\end{algorithmic}
\end{algorithm}

\subsection{The experiments on more complex non-linear relationships}
\label{app:nonliearn}

\textbf{Datasets}. We test our algorithm on synthetic data generated from a \emph{additive non-linear noise model} (Eq.~\ref{eq:data}) with Defs. \ref{def:ALE}, \ref{def:CTT} and \ref{def:ASCG}. For a fixed number of nodes $d$ and edges $e$, we generate the causal graph, represented by a DAG $\mathcal{G}$, using the Erdos-Renyi model \cite{erdos2011evolution}. In this experiments, we generate the data with Gaussian Noise for every variable $X_i^{\tau}$, $i=1,2,\cdots, d$ at time $\tau = 1,2,\cdots, t$:
\begin{eqnarray}
X_i^{\tau}=f_i\left(\mathrm{pa}(X_i^{\tau})\right)+g_i\left(X_i^{\tau-1}\right)+\epsilon_i^{\tau}, \quad \boldsymbol{X}^0 \sim \mathcal{N}\left(0, \mathrm{I}_{d}\right), \quad \boldsymbol{\epsilon}^{\tau} \sim  \mathcal{N}\left(0, 0.4 \cdot \mathrm{I}_{d}\right),
\end{eqnarray} 
where $f_i$ is a twice continuously differentiable arbitrary function in each component, $g_i$ is an arbitrary function for $X_i^{\tau-1}$, and $\mathrm{I}_{d}$ is a $d$ order identity matrix. 
To simulate real-world data as much as possible, we design 3 different twice continuously differentiable non-linear functions $\text{non-linear}(\cdot)$ to discuss the performance of the HTS-CIT algorithm:
\begin{eqnarray}
\label{eq:nonlinear}
\text{Sin} (\mathrm{pa}(X_i^{\tau})) & = & \sum_{j \in \mathrm{pa}(X_i)} \sin (X_j ^{\tau}), \\
\text{Sigmoid} (\mathrm{pa}(X_i^{\tau})) & = & \sum_{j \in \mathrm{pa}(X_i)} \frac{3}{1+\exp{(-X_j^\tau)}}, \\
\text{Poly} (\mathrm{pa}(X_i^{\tau})) & = & \sum_{j \in \mathrm{pa}(X_i)} \frac{1}{10} \left(X_j ^{\tau} + 2\right)^2.
\end{eqnarray} 
In this paper, we use Sin-$d$-$e$ to denote the synthetic dataset generated by non-linear function $\text{Sin}(\cdot)$ with $d$ nodes and $e$ edges: 
\begin{eqnarray}
X_i^{\tau}=\text{Sin}\left(\mathrm{pa}(X_i^{\tau})\right)+\text{Sin}\left(X_i^{\tau-1}\right)+\epsilon_i^{\tau}. 
\end{eqnarray} 
Similarly, we define Sigmoid-$d$-$e$ and Poly-$d$-$e$.

\textbf{Results}. To simulate real-world data as much as possible, we design 2 additional non-linear functions to test the performance of our HT-CIT, i.e., Sigmoid-$d$-$e$ \& Poly-$d$-$e$. The results (Table \ref{tab:Table5}) demonstrate that our HT-CIT remains superior for other complex nonlinear functions with low error edges in identifying causal graphs. In addition, the number of spurious edges that must be pruned in the hierarchical topological ordering is also minimal compared to CAM and SCORE. 

\begin{table}
  \caption{The experiments on Sigmoid-10-10 \& Poly-10-10 with observatios ( $\mathcal{D} = \{\boldsymbol{X}^1, \boldsymbol{X}^2\}$ )}
  \label{tab:Table5}
  \centering
  \scalebox{0.82}{
  \begin{tabular}{c|ccccc}
    \toprule
    & \multicolumn{5}{c}{\bf Sigmoid-10-10 data with observational data ( $\mathcal{D} = \{\boldsymbol{X}^1, \boldsymbol{X}^2\}$ )}\\
    \midrule
    Method & \textbf{SHD$\downarrow$} & \textbf{SID$\downarrow$}& \textbf{F1-Score$\uparrow$} & \textbf{Dis.$\downarrow$} & \textbf{\#Prune$\downarrow$}  \\
    \midrule
    GOLEM & 4.30$_{\pm 2.19}$ & 18.4$_{\pm 7.92}$ & 0.78$_{\pm 0.11}$ & 2.00$_{\pm 0.51}$ & - \\
    NOTEARS & 12.5$_{\pm 5.40}$ & 45.3$_{\pm 17.9}$ & 0.46$_{\pm 0.21}$ & 3.44$_{\pm 0.78}$ & - \\
    ReScore & 12.2$_{\pm 4.30}$ & 45.6$_{\pm 14.4}$ & 0.45$_{\pm 0.17}$ & 3.43$_{\pm 0.63}$ & - \\
    CAM & 3.70$_{\pm 3.43}$ & 10.4$_{\pm 7.86}$ & 0.82$_{\pm 0.17}$ & 1.55$_{\pm 1.20}$ & 80.00$_{\pm 0.00}$ \\
    SCORE & 9.90$_{\pm 3.81}$ & 32.8$_{\pm 11.6}$ & 0.56$_{\pm 0.16}$ & 3.09$_{\pm 0.61}$ & 38.90$_{\pm 1.60}$ \\
    \midrule
    \textbf{HT-CIT} & \bf 0.67$_{\pm 1.12}$ & \bf 1.80$_{\pm 2.99}$ & \bf 0.96$_{\pm 0.06}$ & \bf 0.46$_{\pm 0.72}$ & \bf 8.67$_{\pm 2.92}$ \\
    \bottomrule
    \toprule
    Method & \multicolumn{5}{c}{\bf Poly-10-10 data with observational data ( $\mathcal{D} = \{\boldsymbol{X}^1, \boldsymbol{X}^2\}$ )}\\
    \midrule
    GOLEM & 19.00$_{\pm 4.00}$ & 59.4$_{\pm 13.6}$ & 0.20$_{\pm 0.12}$ & 4.33$_{\pm 0.45}$ & - \\
    NOTEARS & 17.8$_{\pm 5.36}$ & 56.4$_{\pm 16.9}$ & 0.23$_{\pm 0.18}$ & 4.16$_{\pm 0.64}$ & - \\
    ReScore & 17.7$_{\pm 4.73}$ & 57.3$_{\pm 14.1}$ & 0.22$_{\pm 0.15}$ & 4.16$_{\pm 0.56}$ & - \\
    CAM & 8.00$_{\pm 4.69}$ & 19.8$_{\pm 7.88}$ & 0.63$_{\pm 0.21}$ & 2.68$_{\pm 0.95}$ & 80.00$_{\pm 0.00}$ \\
    SCORE & 18.90$_{\pm 4.33}$ & 40.4$_{\pm 10.9}$ & 0.23$_{\pm 0.13}$ & 4.32$_{\pm 0.52}$ & 42.20$_{\pm 1.48}$ \\
    \midrule
    \textbf{HT-CIT} & \bf 3.22$_{\pm 3.15}$ & \bf 10.8$_{\pm 5.69}$ & \bf 0.84$_{\pm 0.15}$ & \bf 1.51$_{\pm 1.03}$ & \bf 11.33$_{\pm 3.87}$ \\
  \bottomrule
  \end{tabular}
  }
\end{table}

\subsection{The experiments on large graph with high-dimension variables}
\label{app:high-dimension}

\textbf{Datasets}. Followed the data generation process (Eq.~\eqref{eq:dataEX14}) in Section \ref{sec:complex} in the main text.
For a fixed number of nodes $d$ and edges $e$, we generate the causal graph, represented by a DAG $\mathcal{G}$, using the Erdos-Renyi model. 
\begin{eqnarray}
X_i^{\tau}=\text{Sin}\left(\mathbf{pa}_{i}^{\tau}\right)+\text{Sin}\left(X_i^{\tau-1}\right)+\epsilon_i^{\tau}, \quad \boldsymbol{X}^0 \sim \mathcal{N}\left(0, \mathrm{I}_{d}\right), \quad \boldsymbol{\epsilon}^{\tau} \sim  \mathcal{N}\left(0, 0.4 \cdot \mathrm{I}_{d}\right)
\end{eqnarray} 
where $\text{Sin} (\mathbf{pa}_{i}^{\tau}) = \sum_{j \in \mathrm{pa}(X_i)} \sin (X_j ^{\tau})$, and $\mathrm{I}_{d}$ is a $d$ order identity matrix. 
To evaluate our HT-CIT on a larger graph with 50/100 nodes, we vary the number of nodes ($d$) and edges ($e$) of the sampled graph and generate Sin-50-50 and Sin-100-100.

For the CIVs, although theoretically, HT-CIT can achieve unbiased estimation, it is limited by the performance of conditional independence tests. 
For the conditional instrumental variables described above, we calculate the conditional independencies using the conditional independence HSIC test from \cite{DBLP:conf/uai/ZhangPJS11} with Gaussian kernel. However, as the data dimension increases, the accuracy of the HSIC test decreases, leading to incorrect topological orderings generated by HT-CIT. To mitigate this issue, given two-time slices  ( $\mathcal{D} = \{\boldsymbol{X}^1, \boldsymbol{X}^2\}$ ), we implement random intervention to some nodes in the previous states of two time-slices, and then apply HT-CIT to identify potential directed acyclic graphs. Based on the percentage of intervened nodes in the previous state, we refer to it as \textbf{HT-CIT(50\% Intervention)}, \textbf{HT-CIT(80\% Intervention)}, and \textbf{HT-CIT(100\% Intervention)}. Note that the time consumption of NOTEARS-MLP and CAM increases substantially (exceeds 5000s) as the graph size increases, thus, in this experiments, we do not discuss NOTEARS-MLP and CAM.

\textbf{Results}. 
From the results on larger graphs (Sin-50-50 and Sin-100-100) in Table \ref{tab:Table6}, we have the following observation: 
(1) GOLEM and ReScore fails to identify the ture DAG on larger graphs; (2) In terms of pruning efficiency, HT-CIT outperforms SCORE by providing a smaller hierarchical topological ordering. The number of edges to be pruned in the topological ordering learned by SCORE is at least 10 times greater that of the proposed HT-CIT, which greatly increases the workload for subsequent pruning processes;
(3) With at least 50\% intervened nodes in the previous stats, HT-CIT(50\% Intervention) can produce results that are comparable to the most advanced methods SCORE. As the proportion of intervened nodes in the previous stats increases (exceed 80\%), our approach (HT-CIT(80\% Intervention) and HT-CIT(100\% Intervention)) will gradually outperform SCORE. Two time-slices with random intervention will help to improve the identification of the topological ordering of the underlying DAG.

\begin{table}
  \caption{The experiments on Sin-50-50 \& Sin-100-100 datasets. }
  \label{tab:Table6}
  \centering
  \resizebox{\linewidth}{!}{
  \begin{tabular}{c|ccccc|c}
    \toprule
    & \multicolumn{6}{c}{\bf Sin-50-50 data with Gauss noise ( $\mathcal{D} = \{\boldsymbol{X}^1, \boldsymbol{X}^2\}$ )} \\
    \midrule
    Method & \textbf{SHD$\downarrow$} & \textbf{SID$\downarrow$} & \textbf{F1-Score$\uparrow$} & \textbf{Dis.$\downarrow$} & \textbf{\#Prune$\downarrow$} & \textbf{Running Time(s)$\downarrow$} \\
    \midrule
    GOLEM & 87.9$_{\pm 11.0}$ & 846.5$_{\pm 166}$ & 0.24$_{\pm 0.10}$ & 9.35$_{\pm 0.58}$ & - & 1049.1s\\
    % NOTEARS& 1.60$_{\pm 0.06}$ & 3.70$_{\pm 3.10}$ & 0.92$_{\pm 0.03}$ & 1.23$_{\pm 0.26}$ & - \\
    ReScore & 83.5$_{\pm 7.17}$ & 1044$_{\pm 65.4}$ & 0.31$_{\pm 0.06}$ & 9.13$_{\pm 0.39}$ & - & 455.2s \\
    SCORE & 17.4$_{\pm 6.17}$ & 91.4$_{\pm 49.7}$ & 0.85$_{\pm 0.04}$ & 4.11$_{\pm 0.74}$ & 1175$_{\pm 0.32}$ & \bf 143.3s\\
    \midrule
    \bf HT-CIT(50\% Intervention) & 21.3$_{\pm 6.85}$ & 102.$_{\pm 64.3}$ & 0.81$_{\pm 0.05}$ & 4.56$_{\pm 0.64}$ & 175.$_{\pm 8.17}$ & 829.1s \\
    \bf HT-CIT(80\% Intervention) & 18.5$_{\pm 5.62}$ & 97.2$_{\pm 39.0}$ & 0.84$_{\pm 0.05}$ & 4.11$_{\pm 0.57}$ & 86.6$_{\pm 7.52}$ & 625.7s\\
    \bf HT-CIT(100\% Intervention) & \bf 16.4$_{\pm 4.40}$ & \bf 88.4$_{\pm 37.1}$ & \bf 0.86$_{\pm 0.04}$ & \bf 4.00$_{\pm 0.58}$ & \bf 58.8$_{\pm 8.52}$ & 327.1s \\
    \bottomrule
    \toprule
    Method  & \multicolumn{6}{c}{\bf Sin-100-100 data with Gauss noise ( $\mathcal{D} = \{\boldsymbol{X}^1, \boldsymbol{X}^2\}$ )} \\
    \midrule
    GOLEM & 160.6$_{\pm 17.2}$ & 1898$_{\pm 764.1}$ & 0.25$_{\pm 0.08}$ & 12.6$_{\pm 0.68}$ & - & 3904.2s\\
    % NOTEARS & 2.00$_{\pm 1.34}$ & 4.80$_{\pm 1.30}$ & 0.91$_{\pm 0.05}$ & 1.29$_{\pm 0.57}$ & - \\
    ReScore  & 163.3$_{\pm 13.7}$ & 4009$_{\pm 549.3}$ & 0.34$_{\pm 0.05}$ & 12.7$_{\pm 0.55}$ & - & 578.2s\\
    SCORE & 32.6$_{\pm 4.71}$ & 211.4$_{\pm 39.7}$ & 0.86$_{\pm 0.02}$ & 5.70$_{\pm 0.42}$ & 4450$_{\pm 0.47}$ & \bf 149.2s \\
    \midrule
    \bf HT-CIT(50\% Intervention) & 37.6$_{\pm 6.30}$ & 219.4$_{\pm 64.3}$ & 0.83$_{\pm 0.02}$ & 6.11$_{\pm 0.49}$ & 306.4$_{\pm 17.6}$ & 3204.2s \\
    \bf HT-CIT(80\% Intervention) & 29.5$_{\pm 6.26}$ & 180.9$_{\pm 47.2}$ & 0.87$_{\pm 0.02}$ & 5.40$_{\pm 0.56}$ & 248.0$_{\pm 27.0}$ & 1187.2s \\
    \bf HT-CIT(100\% Intervention) & \bf 26.7$_{\pm 6.80}$ & \bf 160.5$_{\pm 55.3}$ & \bf 0.88$_{\pm 0.03}$ & \bf 5.13$_{\pm 0.63}$ & \bf 197.1$_{\pm 24.8}$ & 761.2s \\
    \bottomrule
  \end{tabular}
  }
\end{table}

\subsection{Real-world dataset}
\label{app:real}

\textbf{PM-CMR}\footnote{PM-CMR:https://pasteur.epa.gov/uploads/10.23719/1506014/SES\_PM25\_CMR\_data.zip} \cite{wyatt2020dataset} study the impact of $PM_{2.5}$ particle level on the cardiovascular mortality rate (CMR) in 2132 counties in the US using the data provided by the National Studies on Air Pollution and Health. As a real application, we use the 9 variables ($PM_{2.5}$ ($T$), CMR ($Y$), unemployment ($X_1$), income ($X_2$), female householder ($X_3$), vacant housing ($X_4$), owner-occupied housing ($X_5$), educational attainment ($X_6$), and poverty families ($X_7$)) in 2000 \& 2010 as observations for causal discovery. The corresponding description of each variable is detailed in Table~\ref{tab:description4}.

\textbf{Results}. We apply GOLEM \cite{ng2020golem}, SCORE \cite{rolland2022score} and the proposed HT-CIT to public PM-CMR data in 2000 \& 2010, and recover a CMR-related causal graph. As illustrated in Fig.~\ref{fig:figure2} and \ref{fig:figure6}, it can be inferred that the DAG recovered by HT-CIT contains more accurate information. In prior studies \cite{wyatt2020dataset}, $\{X_1, X_2, \cdots, X_7\}$ were thought to be confounders in the causal relation of $T$ to $Y$. However, in Figs. \ref{fig:figure2}(a-b), GOLEM shows there are no direct edge from $T$ to $Y$ and SCORE show that $T$ is the parent node of $\{X_1, X_5, X_6\}$, which contradicts previous research, which seems inconsistent with the settings in causal effect literature.

In the experiments on denser graph with more edges ($e=2d$ and $e=3d$), we selectively report on a few of the best performing baselines on observational data ($\mathcal{D} = \{\boldsymbol{X}^1, \boldsymbol{X}^2\}$) in Table \ref{tab:Table2}. Most previous baselines were only applicable to sparse graphs, whereas our algorithm exhibits substantial improvements on dense graphs. In comparison to the best baseline, our algorithm boasts a 48\% increase in SHD, a 48\% increase in SID, and a 15\% boost in F1-Score on Sin-10-20, and boasts a 30\% increase in SHD, a 43\% increase in SID, and a 7\% boost in F1-Score on Sin-10-30. 

Notably, on denser graphs (Table \ref{tab:Table2}), our algorithm demonstrates significant improvement. In comparison to the best baseline, our algorithm boasts a 48\% increase in SHD, a 48\% increase in SID, and a 15\% boost in F1-Score on Sin-10-20, and boasts a 30\% increase in SHD, a 43\% increase in SID, and a 7\% boost in F1-Score on Sin-10-30.  Most previous baselines were only applicable to sparse graphs, whereas our algorithm exhibits substantial improvements on dense graphs.  Therefore, we believe that HT-CIT provides a more precise DAG for the PM-CMR dataset. Therefore, to effectively combat cardiovascular disease, it is recommended for cities to disseminate information about its dangers, promote prevention, and provide medical care for low-income families.

\begin{table}[t]
\caption{The Description for Real Variables on PM-CMR Dataset.}
\label{tab:description4}
\vskip 0.15in
\centering
\resizebox{\linewidth}{!}{
\begin{tabular}{lc}
\toprule
Variable & Description  \\
\midrule
$PM_{2.5}$($T$) & Annual county PM2.5 concentration, $\mu \mathrm{g} / \mathrm{m}^3$ \\
CMR($Y$) & Annual county cardiovascular mortality rate, deaths/100,000 person-years \\
\midrule
Unemploy($X_1$) & Civilian labor force unemployment rate in 2010 \\
Income($X_2$) & Median household income in 2009 \\
Female($X_3$) & Family households - female householder, no spouse present in 2010  / Family households in 2010 \\
Vacant($X_4$) & Vacant housing units in 2010 / Total housing units in 2010 \\
Owner($X_5$) & Owner-occupied housing units - percent of total occupied housing units in 2010 \\
Edu($X_6$) & Educational attainment - persons 25 years and over - high school graduate (includes equivalency) in 2010 \\
Poverty($X_7$) & Families below poverty level in 2009 \\
\bottomrule
\end{tabular}
}
\end{table}

\begin{figure}[t]
\centerline{\includegraphics[width=0.9\columnwidth]{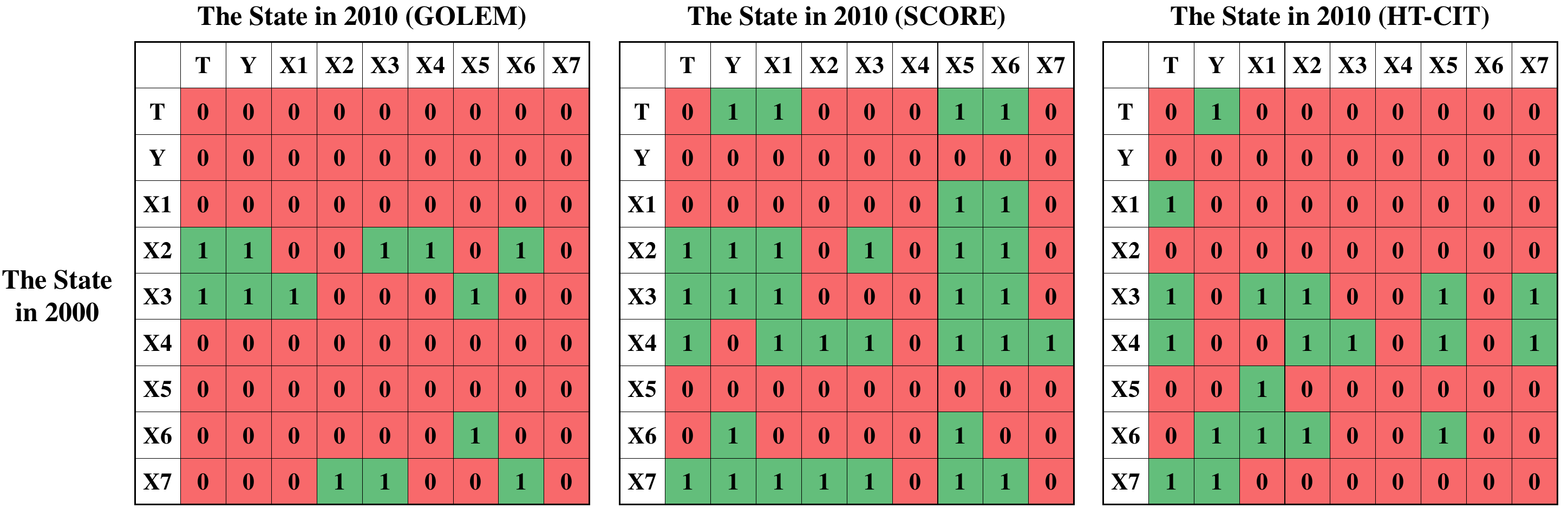}}
\caption{Learned Adjacency Matrix on PM-CMR Dataset.}
\label{fig:figure6}
\end{figure}

\end{document}